\documentclass{article} 
\usepackage{iclr2027_conference,times}

\usepackage{amsmath,amsfonts,bm}

\def\1{\bm{1}}

\DeclareMathAlphabet{\mathsfit}{\encodingdefault}{\sfdefault}{m}{sl}
\SetMathAlphabet{\mathsfit}{bold}{\encodingdefault}{\sfdefault}{bx}{n}

\usepackage{hyperref}
\usepackage{url}
\usepackage{graphicx}
\usepackage{colortbl}
\usepackage[most]{tcolorbox}
\usepackage{fvextra}
\usepackage{subcaption}
\usepackage{enumitem}
\usepackage{amssymb}
\usepackage{mathtools}
\usepackage{amsthm}
\usepackage{amsmath}
\usepackage{bbm}

\theoremstyle{plain}
\newtheorem{theorem}{Theorem}[section]
\newtheorem*{theorem*}{Theorem}

\newtheorem{lemma}[theorem]{Lemma}
\newtheorem*{lemma*}{Lemma}

\theoremstyle{definition}
\newtheorem{definition}[theorem]{Definition}

\theoremstyle{remark}

\newcommand{\G}{\mathcal{G}}
\newcommand{\bs}[1]{\boldsymbol{#1}}

\newcommand{\algname}{\textnormal{\textsc{HarA}}}

\title{Hierarchical Credit Assignment for RLVR on Fused Gromov-Wasserstein Geometry}

\author{%
  Qi Yu$^1$, Ruizhong Qiu$^1$, Zhichen Zeng$^1$, Xuying Ning$^1$, Yanjun Zhao$^1$ \\
  \textbf{Dongqi Fu$^2$, Yinglong Xia$^2$, Hong Li$^2$, Hanghang Tong$^1$}\\
  $^1$University of Illinois Urbana-Champaign, $^2$Meta
}

\iclrpreprint 
\begin{document}

\maketitle


\begin{abstract}
    Reinforcement learning with verifiable rewards (RLVR) has been shown to improve the reasoning capability of large language models (LLMs) across diverse reasoning tasks.
    However, group-based RLVR methods, such as GRPO, assign a uniform advantage to all tokens within rollouts of the same outcome.
    While existing works refine credit assignment of GRPO based on local signals such as token positions or entropy, they often fail to capture the {\em global semantic novelty} of a reasoning behavior relative to the current policy.
    In this work, we propose a \underline{h}ierarchic\underline{a}l c\underline{r}edit \underline{a}ssignment approach for group-based RLVR methods, called \algname,
    which identifies and exploits semantically novel reasoning behaviors during RLVR.
    \algname\ represents each sampled rollout as a distribution over the hidden states and positions of tokens,
    and computes the Fused Gromov-Wasserstein (FGW) barycenters of all rollouts with the same outcome, capturing the internal reasoning patterns in the latent space under the current policy.
    The semantic novelty of a reasoning element can then be measured by its contribution to the FGW distance between the current rollout and the barycenter.
    While solving the FGW formulation is expensive, we introduce an anchor-guided linearization that turns it into a Wasserstein formulation solvable via the Sinkhorn algorithm efficiently.
    By reweighing token-level advantage of group-based RLVR methods based on the novelty signals, \algname\ exploits novel reasoning behaviors at flexible granularities to encourage fine-grained LLMs exploration.
    Extensive experiments across three group-based RLVR methods show that our plug-and-play method effectively enhances the exploration of LLMs, outperforming existing methods across diverse reasoning benchmarks.
\end{abstract}

\section{Introduction}\label{sec:intro}

Recent advances in large language models (LLMs) have revealed remarkable potentials of reinforcement learning with verifiable rewards (RLVR)~\citep{schulman2017proximal, shao2024deepseekmath} in empowering LLMs with complex reasoning capabilities.
Group-based RLVR methods~\citep{shao2024deepseekmath, liu2025understanding, yu2026dapo}, such as Group Relative Policy Optimization (GRPO)~\citep{shao2024deepseekmath}, samples multiple rollouts to estimate their relative advantages by grouping them based on binary rewards from a verifier.
Such approaches eliminate the need for training an explicit critic model, and has become a major paradigm for RLVR.
~\citep{zhang2025survey, zhang2025landscape, wei2026agentic}.

Despite the effectiveness of group-based RLVR methods, they assign uniform advantage to different sampled rollouts and their consisting tokens due to the sparsity of verifiable rewards,
providing coarse-grained credit assignment that limits the exploration of LLMs during post-training.
While there exists works which aim to better characterize rollout and/or token-level credits to encourage better exploration,
they bear the following key limitations:
\textbf{1) Locality of the signals.} Existing credit assignment methods for GRPO mostly rely on local statistics within individual rollouts,
such as token position~\citep{parthasarathi2025grpo} or entropy~\citep{cheng2026reasoning, xie2026unlocking}, for advantage reshaping.
While advantage estimation can propagate such signals across tokens, it remains local within each rollout thus cannot capture the novelty of a reasoning token globally relative to the current policy.
\textbf{2) Surface-level modeling.} The position and entropy-based uncertainty of a token largely encode the surface-level information of a token,
failing to capture its true semantic novelty.
For example, as shown in Figure~\ref{fig:teaser}, position-based methods emphasize prior tokens which hardly impact future reasoning states,
while entropy-based methods assign higher credits to tokens with flexible wording choices (e.g.,``Thus") rather than tokens that leads to meaningful exploration.
\textbf{3) Lack of multi-scale attribution.} Novel reasoning naturally occurs at multiple scales, from individual tokens, steps, to complete reasoning trajectories.
However, existing credit assignment methods typically operates at a single granularity, assigning rollout~\citep{liu2026leveraging}, span~\citep{chen2026hidden},
or token-level credits~\citep{parthasarathi2025grpo, cheng2026reasoning} without a principled mechanism for attribution across multiple granularities.

\begin{figure*}[t]
  \centering
  \includegraphics[width=\linewidth, trim=-10 5 -10 2, clip]{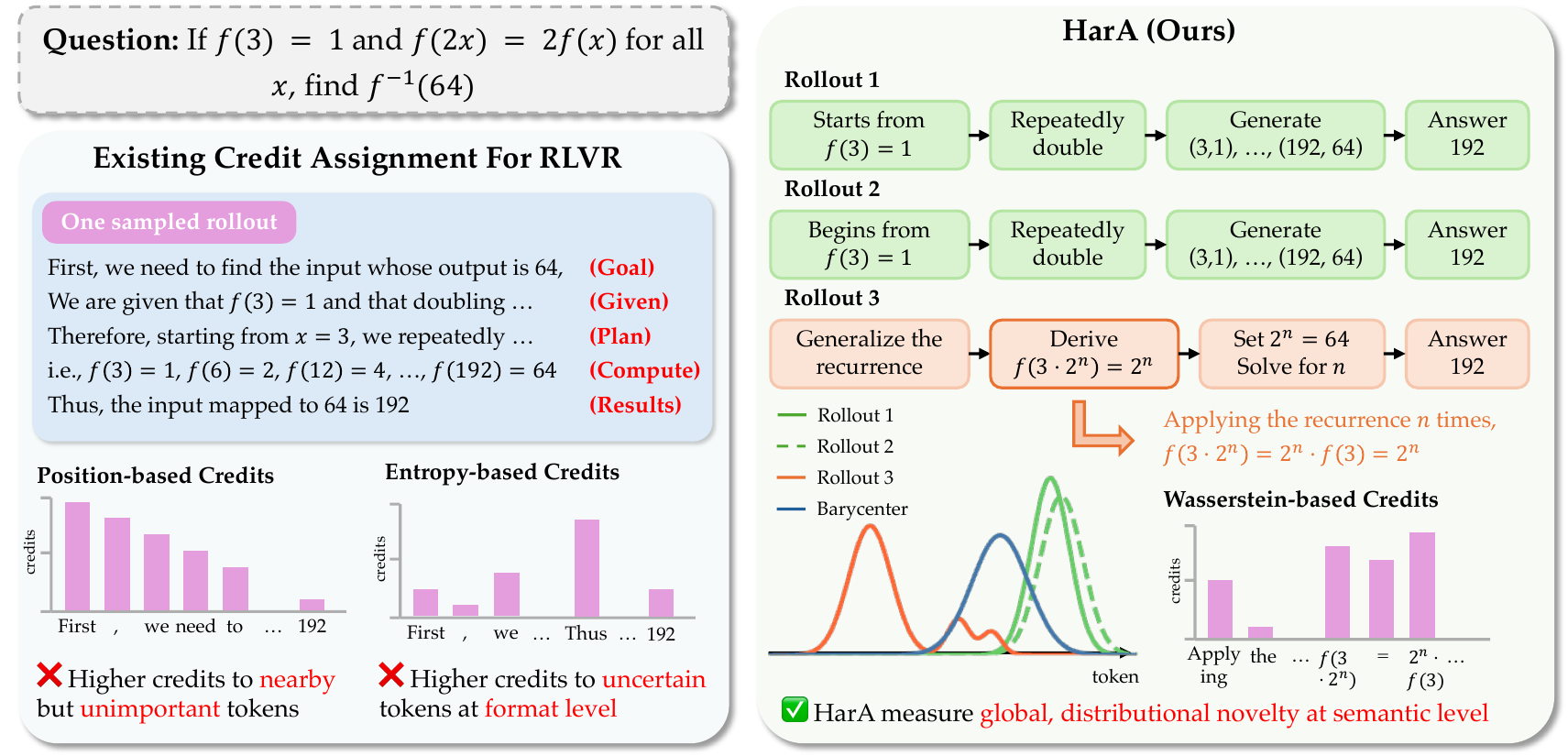}
  \vspace{-15pt}
  \caption{An example comparing existing credit assignment methods and the proposed \algname.}
  \vspace{-30pt}
  \label{fig:teaser}
\end{figure*}

To tackle the above limitations of existing credit assignment methods, we propose a hierarchical credit assignment approach for group-based RLVR called \algname,
leveraging Optimal Transport (OT) as a principled tool for identifying and encouraging reasoning novelty at arbitrary granularity.
\algname\ goes beyond local signals by representing each rollout as a distribution over tokens and comparing it to a Fused Gromov-Wasserstein (FGW) barycenter of all rollouts of the same outcome,
capturing the global novelty of a reasoning element relative to the current policy (\textbf{Limitation \#1}).
Moreover, by encoding the transportation cost via both the hidden states and the structural position of a token,
\algname\ summarizes shared reasoning patterns in the latent space via the FGW barycenter, thus capturing reasoning novelty at both semantical~\citep{hao2024training} and structural levels (\textbf{Limitation \#2}).
While the original FGW formualtion is computationally expensive, we leverage an anchor-guided linearization to turns it into a Wasserstein formulation that can be solved efficiently via the Sinkhorn algorithm~\citep{gabriel2019computational}.
Finally, the FGW formulation naturally provide fine-grained soft alignment of tokens to abstract reasoning steps represented by barycenter supports,
making it easy to measure the novelty of a reasoning element at arbitrary granularity based on its FGW distance to the barycenter (\textbf{Limitation \#3}).
By reweighing token-level advantages based on the novelty measure, \algname\ serves as a direct plug-and-play framework that is widely applicable to diverse group-based RLVR methods~\citep{shao2024deepseekmath, liu2025understanding, yu2026dapo}.

Comprehensive experiments across 3 model sizes, 5 reasoning benchmarks, and 3 group-based RLVR methods demonstrates that \algname\ avoids premature entropy collapse and leads to more extensive exploration,
outperforming existing credit assignment approaches while maintaining a good balance between effectiveness and efficiency.
Moreover, our observation in Figure~\ref{fig:obs} shows that the FGW barycenter encodes internal reasoning patterns of the LLM, making them principle references for identifying novel reasoning behaviors that contribute to explorations of LLMs during RLVR.
The main contributions of this paper are summarized as follows:
\begin{itemize}[nosep, wide=0pt, leftmargin=*, after=\strut, label=\textbullet]
  \item \textbf{Barycenter View of LLMs.} To our best knowledge, we are the first to represent internal reasoning patterns of LLMs via the FGW barycenter of its sampled rollouts in the latent space.
  \item \textbf{Multi-scale Credit Assignment.} We propose a hierarchical credit assignment approach \algname\ for group-based RLVR methods that idenitfy novel reasoning behaviors at arbitrary granularity.
  \item \textbf{Extensive Analysis.} We provide extensive analysis on the approximation error of the anchor-guided linearization and the complexity of \algname.
  \item \textbf{Strong Performance.} Extensive experiments across diverse model sizes, reasoning benchmarks, and group-based RLVR methods demonstrate that \algname\ consistently outperforms exitsing credit assignment approaches by a large margin.
\end{itemize}

\section{Preliminaries}\label{sec:pre}


\vspace{-3pt}
\subsection{Reinforcement Learning with Verifiable Rewards (RLVR)}\label{subsec:rlvr}
\vspace{-3pt}
RLVR has played a critical role in LLMs post-training to empower them with complex reasoning capabilities~\citep{schulman2017proximal, shao2024deepseekmath}.
Formally, let $\pi_{\theta}$ denotes the policy of an LLM parameterized by $\theta$.
Given a prompt $x$ and its ground-truth answer $a$, the LLM generates a response sequence $y=(y_1, y_2,..., y_l)$ of $l$ tokens autoregressively under its policy,
i.e., $\pi_{\theta}(y|x)=\prod_{i=1}^l\pi_\theta(y_i|x,y_{<i})$, and is evaluated based on a reward function $R(y,a)=\mathbbm{1}\{y\equiv a\}$ where $\mathbbm{1}\{\cdot\}$ denotes a binary indictor function of whether the outcome of response $y$ is equivalent to $a$.

\vspace{-3pt}
\subsection{Group-based Policy Optimization}\label{subsec:gbpo}
\vspace{-3pt}
Group-based Policy Optimization, a dominant paradigm in RLVR, estimates the relative advantages of multiple sampled rollouts by grouping them based on verifiable rewards without an explicit critic model~\citep{shao2024deepseekmath, liu2025understanding, yu2026dapo}.
One representative method is Group Relative Policy Optimization (GRPO)~\citep{shao2024deepseekmath}, whose objective is defined as
\[
    \mathbb{E}_{x,\{y_{k}\}_k^K\sim\pi_\theta(\cdot|x)}\!\!\left[\!\frac{1}{K}\sum_{k=1}^K\frac{1}{|y_k|}\sum_{i=1}^{|y_k|}\min\left(r_{k,i}A_{k}, \text{clip}\left(r_{k,i}, 1-\epsilon, 1+\epsilon\right)A_{k}\right)\!\right]
    -\beta\cdot\text{KL}\left(\pi_\theta(\cdot|x)\|\pi_{\text{ref}}(\cdot|x)\right)
\]
where $\{y_k\}_k^K\sim\pi_\theta(\cdot|x)$ denotes $K$ sampled rollouts,
$r_{k,i}=\frac{\pi_{\theta}(y_{k,i}|x,y_{k,<i})}{\pi_{\theta_{\text{old}}}(y_{k,i}|x,y_{k,<i})}$ denotes the importance ratio where $\pi_{\theta_{\text{old}}}$ is the behavior policy from the last policy gradient update,
and $A_k=\frac{R(y_{k}, a)-\text{mean}\left(\{R(y_k, a)\}^K_{k=1}\right)}{\text{std}\left(\{R(y_k, a)\}^K_{k=1}\right)}$ denotes the relative advantage of the $k$-th sampled rollout.
$\text{KL}(\cdot\|\cdot)$ denotes the KL divergence, $\pi_{\text{ref}}$ denotes the reference policy, and $\beta$ denotes the regularization weight.

\vspace{-3pt}
\subsection{Fused Gromov-Wasserstein Barycenter}\label{subsec:fgw}
\vspace{-3pt}
The FGW distance serves as a powerful tool for measuring the discrepancy between geometric data, such as graphs~\citep{titouan2019optimal, zeng2025pave},
by taking into account both node feature and graph structural information. In this work, we represent each LLM rollout as an attributed graph to jointly capture their semantic and structural information.
Formally, an attributed graph $\G=\{\mathbf{A}, \mathbf{X}\}$, where $\mathbf{A}$ denotes the graph adjacency matrix and $\mathbf{X}$ denotes the node feature matrix,
can be represented by a discrete distribution on the product space of node feature and structure, i.e., $\bs{\mu}=\sum_{i=1}^l\mathbf{w}(i)\delta_{\mathbf{A}(v_i), \mathbf{X}(v_i)}$,
where $\mathbf{w}\in\Delta_l$ is a histogram representing the node weight of node $v_i\in\G$, and $\delta$ is the Dirac function.
In this way, the FGW distance, which combines the Wasserstein (W) and the Gromov-Wasserstein (GW) distance linearly, can be formally defined as follows
\begin{definition}\label{def:fgw_dist}
    \textit{Fused Gromov-Wasserstein (FGW) distance}\\
    \textbf{Given:}
    (1) two attributed graphs $\G_1,\G_2$ with distributions $\bs{\mu}_1,\bs{\mu}_2$,
    (2) a structural distance measure $f\!:\!\G_i\!\times\!\G_i\!\to\!\mathbb{R}_{\geq 0}$ which captures pairwise node distances within the same graph based on graph structure $\mathbf{A}_i$ for $i\in\{1,2\}$, e.g., shortest-path distance,
    (3) a feature distance measure $g\!:\!\G_1\!\times\!\G_2\!\to\!\mathbb{R}_{\geq 0}$ which captures pairwise node distances across two graphs based on node features $\mathbf{X}_1$ and $\mathbf{X}_2$, e.g., $L_2$ distance,
    (4) the order of FGW distance $q$,
    (5) hyperparameter $0<\alpha<1$\\
    \textbf{Output:} the FGW distance $\mathrm{FGW}_{q,\alpha}(\G_1,\G_2,\bs{\mu}_1,\bs{\mu}_2,f,g)$ between $\G_1$ and $\G_2$
    \begin{equation}\label{eq:fgw_dist}
        \mathrm{FGW}_{q,\alpha}(\G_1,\G_2,\bs{\mu}_1,\bs{\mu}_2,f,g)=\min_{\mathbf{T}\in\Pi(\bs{\mu}_1,\bs{\mu}_2)}
        (1-\alpha)\mathrm{W}_q(\mathbf{T},g)
        +\alpha\mathrm{GW}_q(\mathbf{T},f)
    \end{equation}
    \vspace{-5pt}
    \[
        \text{s.t.}~
        \mathrm{W}_q(\mathbf{T},g)=\mathbb{E}_{(u,v)\sim\mathbf{T}}\left[g^q(u,v)\right],\quad
        \mathrm{GW}_q(\mathbf{T},f)=\mathbb{E}_{(u,v),(u',v')\sim\mathbf{T}\otimes\mathbf{T}}\left[\left(f(u,u')-f(v,v')\right)^q\right]
    \]
\end{definition}
where $u,u'\in\G_1$ and $v,v'\in\G_2$ are nodes in the two graphs. We adopt $q=2$ throughout the paper.
Given multiple graphs, the FGW barycenter serves as a consensus graph by summarizing their shared patterns in node features and graph structures based on FGW distance. Formally,
\begin{definition}
    \textit{Fused Gromov-Wasserstein (FGW) barycenter}~\citep{zeng2024hierarchical}\\
    \textbf{Given:}
    (1) a set of attributed graphs $\left\{\G_k=\{\mathbf{A}_k,\mathbf{X}_k\}\right\}_{k=1}^K$ with distributions $\left\{\bs{\mu}_k\right\}_{k=1}^K$,
    (2) an intra-graph structural measure $f$,
    (3) an cross-graph feature measure $g$,
    (4) a barycenter distribution $\bs{\mu}_b$ with a fixed weight vector $\mathbf{w}_b\in\Delta_{l_b}$,
    (5) the order of FGW distance $q$,
    (6) hyperparameter $\alpha$\\
    \textbf{Output:} the FGW barycenter $\G_b=\{\mathbf{A}_b,\mathbf{X}_b\}$
    \vspace{-5pt}
    \begin{equation}\label{eq:fgw_barycen}
        \{\mathbf{A}_b,\mathbf{X}_b\}:=
        \mathop{\arg\min}\limits_{\mathbf{A},\mathbf{X}}\sum_{k=1}^K\mathrm{FGW}_{q,\alpha}\left(\{\mathbf{A}_k, \mathbf{X}_k\},\{\mathbf{A}, \mathbf{X}\}, \bs{\mu}_k, \bs{\mu}_b, f, g\right)
    \end{equation}
\end{definition}
\vspace{-15pt}

\section{Methodology}\label{sec:method}
\vspace{-9pt}
In this section, we present the proposed credit assignment approach \algname\ for group-based RLVR methods.
Firstly, we introduce a principled way of representing the internal reasoning pattern of LLM within an outcome-grounded group based on the FGW barycenters, in Section~\ref{subsec:me-bary}.
Then, we introduce a hierarchical credit assignment approach based on the FGW distance to the barycenter, in Section~\ref{subsec:me-ca}.
Finally, we provide theoretical analysis regarding the approximation error and the complexity of \algname, in Section~\ref{subsec:me-ana}.

\vspace{-6pt}
\subsection{Distributional Modeling of Internal Reasoning Pattern}\label{subsec:me-bary}
\vspace{-6pt}

Recent studies~\citep{hao2024training, du2026latent} show that the hidden states of an LLM encode abstract reasoning steps in a continuous latent space and can signal when reasoning behavior diverges.
Such observations imply that novel reasoning behaviors may be discovered by comparing the hidden states of sampled rollouts in the latent space.
In this work, we propose to leverage OT as a principled tool for such comparisons.
Specifically, we leverage the FGW barycenters of sampled rollouts within an outcome-grounded group to encode the internal reasoning pattern of the current policy,
which later serves as a principled reference for measuring the novelty of a reasoning element.

\vspace{-6pt}
\subsubsection{FGW Barycenter as Internal Reasoning Pattern}
\vspace{-6pt}

LLM rollouts are in the form of natural language, making principled comparison across different rollouts challenging and costly.
To address this, we project each rollout into a distribution on the product space of the hidden states and position of its tokens,
jointly capturing the semantic and structural information of a rollout in the latent space.
Formally, let $y$ denotes a rollout of $l$ tokens, we represent $y$ as an attributed chain graph $\G=\{\mathbf{A}, \mathbf{X}\}$ with $l$ nodes,
where $\mathbf{A}[i,j]=\mathbbm{1}\{|i-j|=1\}$ represents its chain structure,
and the node feature $\mathbf{X}=\mathbf{H}\in\mathbb{R}^{l\times d}$ is the hidden state matrix of $y$.

After representing rollouts as distributions,
we leverage FGW barycenters to summarize the internal reasoning patterns of the current policy within an outcome-grounded group. Formally, for a prompt $x$ with ground-truth answer $a$,
let $\{\G_k^{(c)}=\{\mathbf{A}_k^{(c)},\mathbf{X}_k^{(c)}\}\}_{k=1}^K$ denotes the graph representations of $K$ sampled rollouts $\{y_k^{(c)}|R(y_k^{(c)},a)=\mathbbm{1}\{c=+\}\}_{k=1}^K$ within group $c\in\{+,-\}$,
i.e., sets of correct and incorrect rollouts,
we compute a FGW barycenter $\G_b^{(c)}=\{\mathbf{A}_b^{(c)},\mathbf{X}_b^{(c)}\}$ for group $c$ as follows
\vspace{-5pt}
\begin{equation}\label{eq:fgw_state}
    \mathbf{X}_b^{(c)}:=\mathop{\arg\min}\limits_{\mathbf{X}}\sum_{k=1}^K\mathrm{FGW}_{2,\alpha}\left(\{\mathbf{A}_k^{(c)},\mathbf{X}_k^{(c)}\},\{\mathbf{A}_b^{(c)},\mathbf{X}\},f,g\right)
\end{equation}

\vspace{-12pt}
where $f(y_{k,i}^{(c)}, y_{k,j}^{(c)})=|i-j|/(l_{k}^{(c)}-1)$ defines a normalized shortest-path distance of tokens within each rollout,
and $g(y_{k,i}^{(c)}, y_{b,j}^{(c)})=\|\mathbf{X}_{k}^{(c)}[i,:]-\mathbf{X}_b^{(c)}[j,:]\|_2$ defines a $L_2$ distance between hidden states of tokens in sampled rollouts and the barycenter.
Note that compared to Eq.~\eqref{eq:fgw_barycen}, we fix $\mathbf{A}_b^{(c)}[i,j]=\mathbbm{1}\{|i-j|=1\}$ to a chain structure to represent a \textit{mean rollout} in the latent space,
and only optimize the hidden states matrix $\mathbf{X}_b^{(c)}$ of the barycenter in Eq.~\eqref{eq:fgw_state}.

However, the FGW barycenter is inherently a discrete distribution that lies in a latent space,
making it less interpretable than an actual rollout consisting of natural language tokens.
To gain a deeper insight into what the FGW barycenter represents,
we conduct the following interpretability studies to examine whether its individual supports correspond to meaningful reasoning stages in natural language.
Specifically, we obtain an OT transport plan $\mathbf{T}^{(c)}_{k}$ for $y_k^{(c)}$ when solving Eq.~\eqref{eq:fgw_state},
which serves as a soft mapping between tokens in $y_k^{(c)}$ and the barycenter.
To further establish a connection between barycenter supports and meaningful reasoning stages, we prompt a powerful LLM, i.e., GPT-5-Sol,
to derive a set of shared reasoning stages $\{S^{(c)}_m\}_{m=1}^M$ from all sampled rollouts in group $c$ in natural languages,
and classify each token to one reasoning stages, e.g., $y_{k,i}^{(c)}\in S^{(c)}_m$. In this way, we obtain a mapping $\mathbf{P}^{(c)}$ between barycenter supports and LLM-derived reasoning stages:
\vspace{-2pt}
\begin{equation}\label{eq:map}
    \mathbf{P}^{(c)}(j,m)
    =
    \frac{
        \sum_{k=1}^{K}\sum_{i=1}^{|y_k^{(c)}|}
        \mathbf{T}_{k}^{(c)}[i,j]
        \mathbbm{1}\!\left\{
        y_{k,i}^{(c)} \in S_m^{(c)}
        \right\}
    }{
        \sum_{k=1}^{K}\sum_{i=1}^{|y_k^{(c)}|}
        \mathbf{T}_{k}^{(c)}[i,j]
    }
\end{equation}
\vspace{-8pt}

\begin{figure*}[t]
    \centering
    \begin{subfigure}{.32\textwidth}
        \centering
        \includegraphics[width=\textwidth, trim=0 0 0 5, clip]{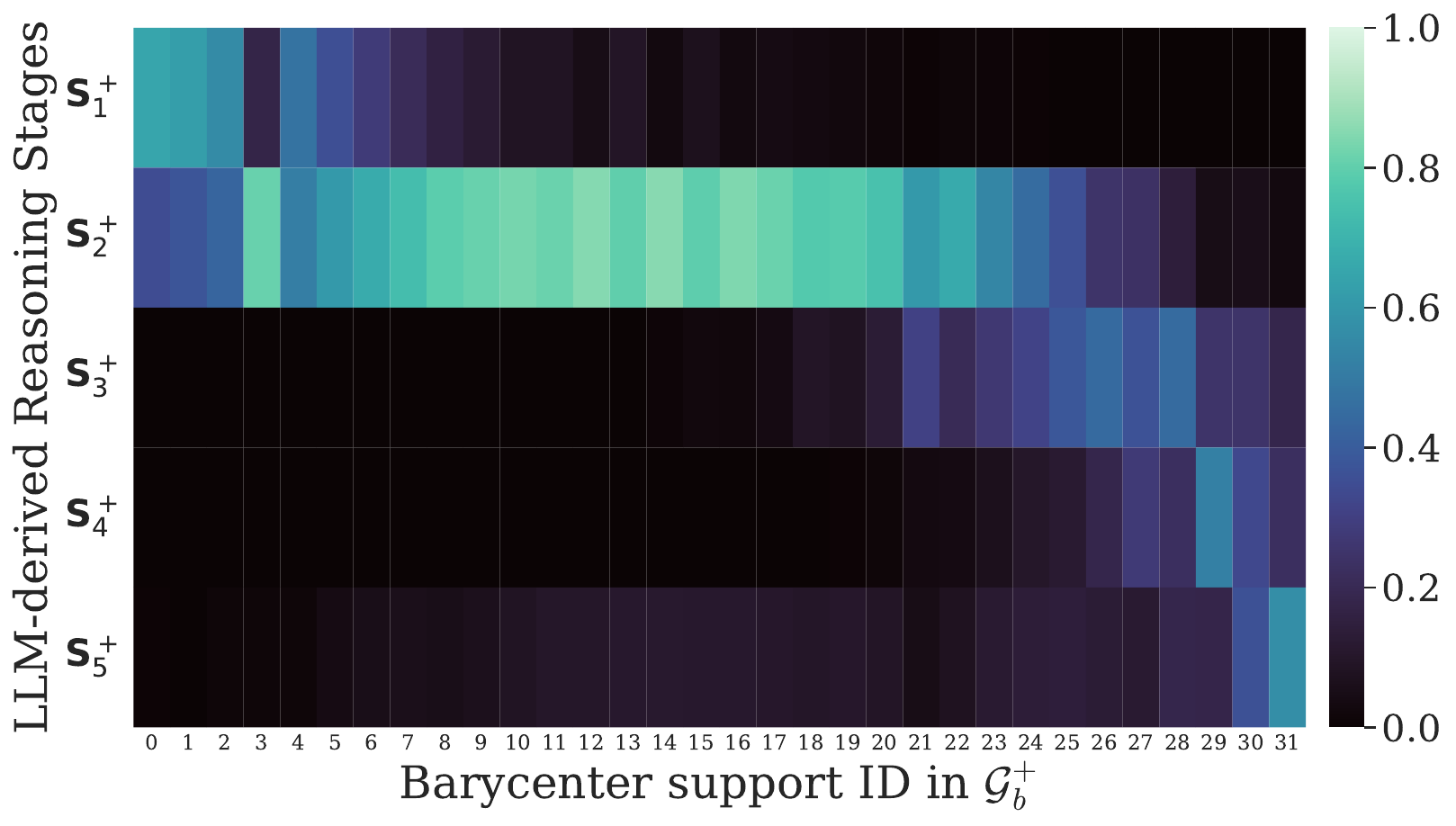}
        \vspace{-15pt}
        \caption{Visualization of mapping $\mathbf{P}^{(+)}$ }
        \label{fig:vis-p-map}
    \end{subfigure}
    \hfill
    \begin{subfigure}{.32\textwidth}
        \centering
        \includegraphics[width=\textwidth, trim=0 5 0 5, clip]{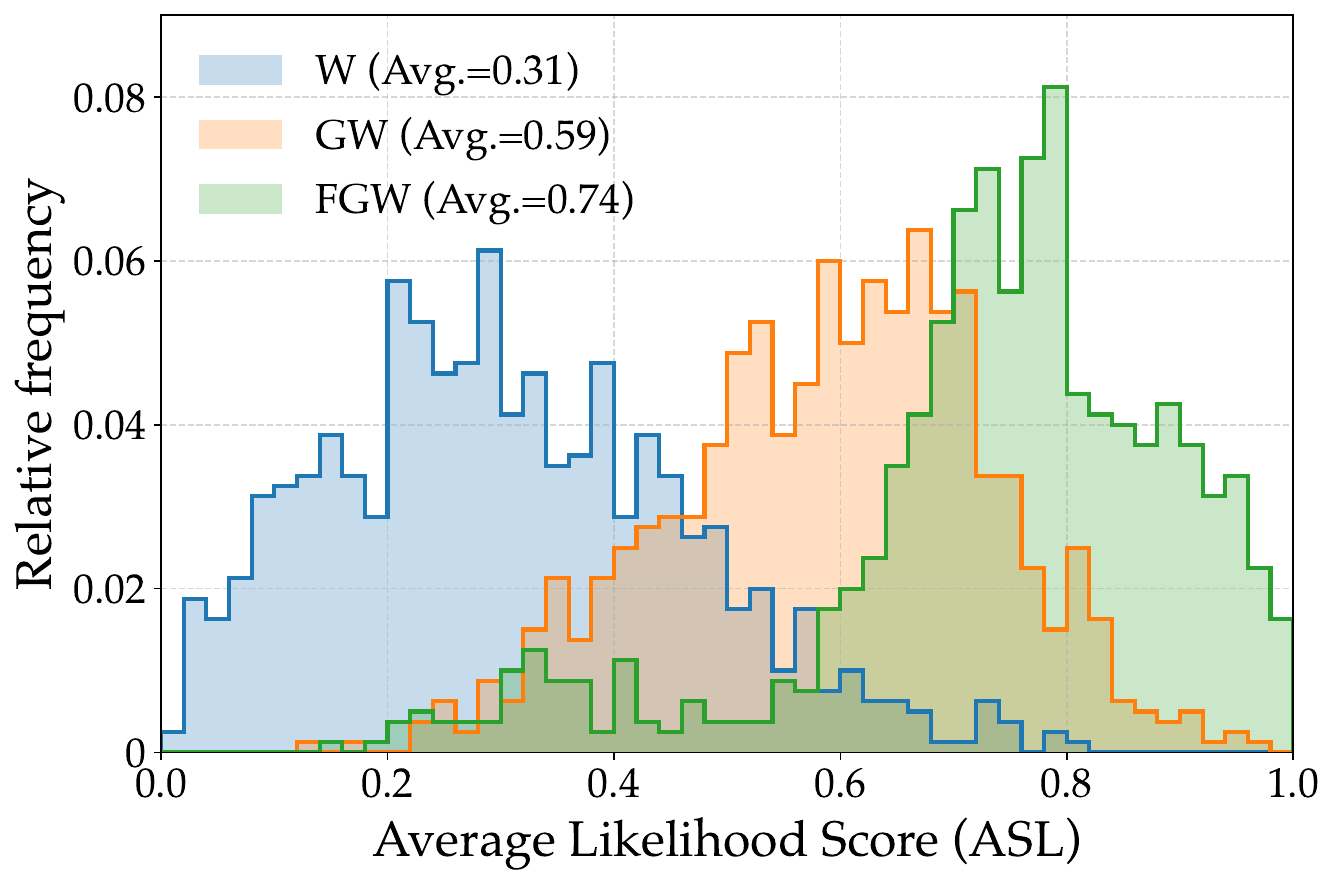}
        \vspace{-15pt}
        \caption{Comparison of OT variants.}
        \label{fig:loro-ot-compare}
    \end{subfigure}
    \hfill
    \begin{subfigure}{.32\textwidth}
        \centering
        \includegraphics[width=\textwidth, trim=0 5 0 5, clip]{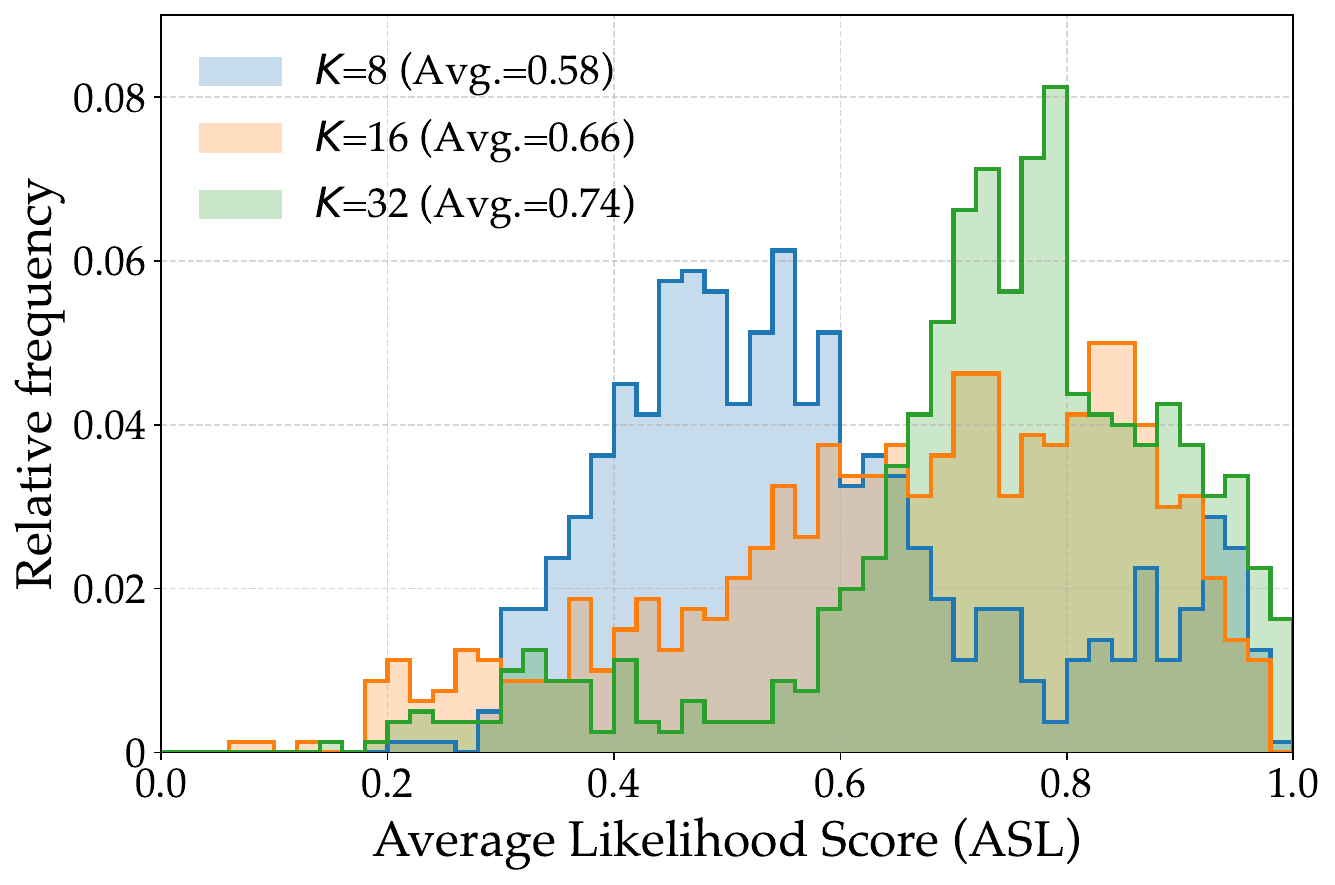}
        \vspace{-15pt}
        \caption{Comparison of sample size $K$.}
        \label{fig:loro-sample-size}
    \end{subfigure}
    \vspace{-10pt}
    \caption{Observations.
        (1) \textbf{Visualization of mapping $\textbf{P}^{(+)}$}: The \textit{sparse} mapping shows that the barycenter supports are mapped deterministically to meaningful reasoning stages.
        (2) \textbf{Comparison of OT variants}: FGW exhibits higher ASL than W and GW, demonstrating stronger expressiveness.
        (3) \textbf{Comparison of sample size $K$}: ASL improves monotonically as $K$ increases, suggesting that the barycenter reliably captures the internal reasoning pattern with sufficient sampling.}
    \vspace{-15pt}
    \label{fig:obs}
\end{figure*}

We visualize $\mathbf{P}^{(+)}$ for one question~\footnote{\textbf{Question}: If $f(3)=1$ and $f(2x)=2f(x)$ for all $x$, find $f^{-1}(64)$} in MATH500~\citep{hendrycks2021measuring} in Figure~\ref{fig:vis-p-map},
which shows a \textit{sparse} and \textit{block-diagonal} matrix indicating that the FGW barycenter supports can be mapped deterministically to meaningful reasoning stages.
To further evaluate the accuracy of such mapping, we design a Leave-One-Rollout-Out (LORO) evaluation where we randomly hold out one sampled rollout $y^{(c)}_e$ when constructing the mapping $\mathbf{P}^{(c)}$,
and then evaluate $\mathbf{P}^{(c)}$ by its average likelihood score (ASL) in predicting token labels in $y^{(c)}_e$ as follows:
\vspace{-2pt}
\begin{equation}
    \text{ASL}(y^{(c)}_e)=\sum_{i=1}^{|y^{(c)}_e|}\sum_{j=1}^{l_b}\sum_{m=1}^M\mathbf{T}^{(c)}[i,j]\mathbf{P}^{(c)}[j,m]\mathbbm{1}\{y^{(c)}_{e,i}\in S^{(c)}_m\}
\end{equation}

\vspace{-5pt}
The comparative results of LORO evaluation across different OT variants and sampling sizes on MATH500 are included in Figure~\ref{fig:loro-ot-compare} and Figure~\ref{fig:loro-sample-size}, respectively.
We observe that \textbf{1)} The FGW formulation achieve higher ASL than the Wasserstein (W) and Gromov-Wasserstein (GW) formulation, demonstrating stronger expressiveness by combining both semantic hidden states and structural positions of tokens;
\textbf{2)} The ASL under FGW formulation improves monotonically as K increases, suggesting that the barycenter reliably captures the internal reasoning pattern of current policy with sufficient sampling.
Detailed experimental settings of Figure~\ref{fig:obs} are included in Appendix~\ref{app:exp_case_setup}.

\vspace{-3pt}
\subsubsection{Anchor-Guided Linearization}
\vspace{-4pt}
The FGW formulation in Eq.~\eqref{eq:fgw_dist}, while expressive, is computationally expensive due to its quartic, non-convex GW term~\citep{peyre2016gromov} of $O(l^3)$ time complexity to solve.
To optimize efficiency, we leverage an anchor-guided linearization technique~\citep{im2026shape} that linearizes the GW term by fixing its second transport plan to pre-align reasoning state anchors.
Specifically, we treat the first and last tokens of each sampled rollout, i.e., $y^{(c)}_{k,1}$ and $y^{(c)}_{k,L_k}$, as reasoning state anchors, as they represent the shared initial and final states of all sampled rollouts under the same prompt $x$ and outcome $c$.
By pre-aligning anchors in the GW term in Eq.~\eqref{eq:fgw_dist}, we can linearize the FGW distance in Eq.~\eqref{eq:fgw_state} to a Wasserstein distance as follows:
\begin{lemma}[Anchor-Guided Linearization of FGW Distance]
    \label{lem:fgw_linear}
    $\mathrm{FGW}_{2,\alpha}$
    \footnote{For brevity, we use $\mathrm{FGW}_{2,\alpha}$ as an abbreviation for $\mathrm{FGW}_{2,\alpha}(\{\mathbf{A}_k^{(c)},\mathbf{X}_k^{(c)}\},\{\mathbf{A}_b^{(c)},\mathbf{X}_b^{(c)}\},f,g)$}
    in Eq.~\eqref{eq:fgw_state} can be linearized into the following Wasserstein distance $\mathrm{FGW}_{2,\alpha}^{\text{linear}}$
    by fixing the second transport plan $\mathbf{T}^{(c)}_{\text{anchor}}=\frac{1}{2}\left(\delta_{(1,1)}+\delta_{(L_k,L_b)}\right)$
    in its $\mathrm{GW}_2$ term that pre-aligns reasoning anchors $y^{(c)}_{k,1}$ and $y^{(c)}_{k,l_k}$.
    \begin{equation}\label{eq:fgw_linear}
        \!
        \mathrm{FGW}_{2,\alpha}^{\mathrm{linear}}
        =\!
        \min_{\mathbf{T}\in\Pi(\bs{\mu}_k,\bs{\mu}_b)}
        \mathbb{E}_{\left(y^{(c)}_{k,i},y^{(c)}_{b,j}\right)\sim\mathbf{T}}
        \left[
            (1-\alpha)g^2\left(y^{(c)}_{k,i},y^{(c)}_{b,j}\right)
            +
            2\alpha
            \left|
            \frac{i-1}{L_k-1}
            -
            \frac{j-1}{L_b-1}
            \right|^2
            \right]
    \end{equation}
\end{lemma}

Detailed proof of Lemma~\ref{lem:fgw_linear} can be found in Appendix~\ref{app:proof_fgw_linearization}.
In general, fixing the second transport plan in $\mathrm{GW}_2$ turns it into a convex Wasserstein term that penalizes the difference in normalized positions of tokens across different rollouts,
which can be solved efficiently via an entropy regularization and the Sinkhorn algorithm~\citep{gabriel2019computational}.
In this way, we can obtain the linearized FGW barycenter efficiently by replacing $\text{FGW}_{2,\alpha}$ in Eq.~\eqref{eq:fgw_state} with $\text{FGW}_{2,\alpha}^{\text{linear}}$ in Eq.~\eqref{eq:fgw_linear}.

\vspace{-3pt}
\subsection{Hierarchical Credit Assignment}\label{subsec:me-ca}
\vspace{-5pt}

\begin{figure*}[t]
    \centering
    \includegraphics[width=\linewidth, trim=-10 20 -10 5, clip]{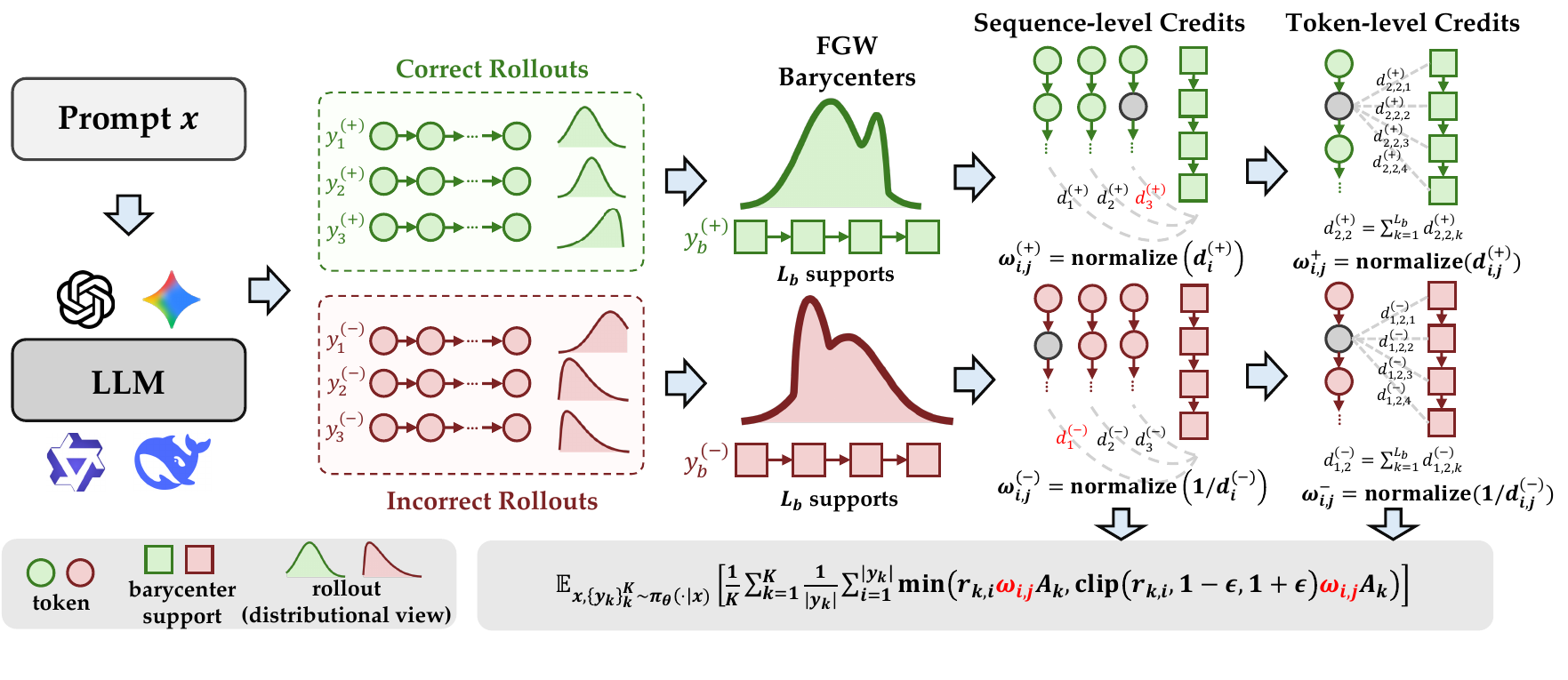}
    \vspace{-20pt}
    \caption{Overview of \algname.
        FGW barycenters are computed from the correct and incorrect rollouts as two reference distributions,
        followed by a hierarchical credit assignment approach that reweights the group advantages at both the sequence and token levels,
        based on their reasoning novelty measured by the linearized FGW distance to the barycenters.}
    \vspace{-10pt}
    \label{fig:method}
\end{figure*}

After obtaining the linearized FGW barycenters as principled reference distributions,
we measure the novelty of a reasoning span at arbitrary granularity by their contributions to the linearized FGW distance to the barycenters,
and propose a hierarchical credit assignment approach for group-based RLVR methods as shown in Figure~\ref{fig:method}.
Formally, let $\mathcal{P}^{(c)}_k$ be a partition of rollout $y_k^{(c)}$ into $M_k^{(c)}=|\mathcal{P}_k^{(c)}|$ reasoning spans,
the novelty-based credit $\omega^{(c)}_{k,\mathcal{S}}$ for a span $y_{k,\mathcal{S}}^{(c)}\subseteq y_k^{(c)}$ is defined as its contribution to the linearized FGW distance to $\G_b^{(c)}$ normalized within group $c$ as follows
\begin{equation}
    \begin{aligned}
        \omega^{(c)}_{k,\mathcal{S}}        & =\frac{\left(\sum_{k'=1}^{K}M_{k'}^{(c)}\right)\left(\sum_{i\in\mathcal{S}}\phi^{(c)}_{k,i}\right)^p}{\sum_{k'=1}^{K}\sum_{\mathcal{S}'\in\mathcal{P}_k^{(c)}}\left(\sum_{i'\in\mathcal{S}'}\phi^{(c)}_{k',i'}\right)^p}, \\
        \text{where} \quad \phi_{k,i}^{(c)} & =\sum_{j=1}^{L_b}\mathbf{T}_k^{(c)}(i,j)\left[(1-\alpha)\left\|\mathbf{X}_{k,i}^{(c)}-\mathbf{X}_{b,j}^{(c),*}\right\|_2^2+2\alpha\left(\frac{i-1}{L_k-1}-\frac{j-1}{L_b-1}\right)^2\right].
    \end{aligned}
\end{equation}
where $p=1$ for correct rollouts (i.e., $c=+$) and $p=-1$ for incorrect rollouts (i.e., $c=-$). The intuition of $p$ is to guide the RLVR process to encourage correct and novel reasoning behaviors more,
while penalizing incorrect but novel reasoning behaviors less to preserve the exploration of the LLM.
Based on the obtained novelty-based credits $\omega^{(c)}_{k,\mathcal{S}}$, we boadcast the same credit to all tokens inside the reasoning span $\mathcal{S}$, i.e., $\omega^{(c)}_{k,i}=\omega^{(c)}_{k,\mathcal{S}}$ for all $i\in\mathcal{S}$, and reweight the token-level group advantages for group-based RLVR methods as follows
\begin{equation}\label{eq:alg-obj}
    \mathbb{E}_{x,\{y_{k}\}_k^K\sim\pi_\theta(\cdot|x)}\left[\frac{1}{K}\sum_{k=1}^K\frac{1}{|y_k|}\sum_{i=1}^{|y_k|}\min\left(r_{k,i}\textcolor{red}{\omega_{k,i}}A_{k}, \text{clip}\left(r_{k,i}, 1-\epsilon, 1+\epsilon\right)\textcolor{red}{\omega_{k,i}}A_{k}\right)\right]
\end{equation}
In this paper, we mainly consider credit assignment on two levels of granularity: \textit{sequence-level} (i.e., $\mathcal{P}^{(c)}_k=\{y^{(c)}_k\}$) and \textit{token-level} (i.e., $\mathcal{P}^{(c)}_k=\{y^{(c)}_{k,i}\}_{i=1}^{l_k}$),
and name the corresponding variants as \algname-seq and \algname-tok, respectively.

\vspace{-3pt}
\subsection{Theoretical analysis}\label{subsec:me-ana}
\vspace{-5pt}
In this subsection, we provide theoretical analysis regarding the approximation error of anchor-guided linearization and the complexity of the proposed \algname.

\begin{theorem}[Approximation Error of Anchor-Guided Linearization]
    \label{thm:linearization_error}
    Let $\mathbf{T}_{\mathrm{FGW}}^*\in\Pi(\mu_k,\mu_b)$ be an optimal transport plan that minimizes the original FGW problem with uniform marginal distributions $\mu_k$ and $\mu_b$.
    The approximation error of the anchor-guided linearization in Lemma~\ref{lem:fgw_linear} is bounded as follows
    \begin{equation}
        0\leq
        \mathrm{FGW}_{2,\alpha}^{\text{linear}}(\mathcal{G}_k^{(c)},\mathcal{G}_b^{(c)})
        -
        \mathrm{FGW}_{2,\alpha}(\mathcal{G}_k^{(c)},\mathcal{G}_b^{(c)})
        \leq
        4\alpha\,\mathcal{C}(\mathbf{T}_{\mathrm{FGW}}^*),
    \end{equation}
    where
    \begin{equation}
        \mathcal{C}(\mathbf{T})
        =
        \mathbb{E}_{\mathbf{T}\otimes\mathbf{T}}
        \left[
            \frac{|I-I'|}{L_k-1}
            \frac{|J-J'|}{L_b-1}
            \mathbbm{1}\!\left\{(I-I')(J-J')<0\right\}
            \right],
        \qquad
        (I,J),(I',J')\overset{\mathrm{i.i.d.}}{\sim}\mathbf{T}.
    \end{equation}
\end{theorem}

The detailed proof of Theorem~\ref{thm:linearization_error} can be found in Appendix~\ref{app:proof_linearization_error}.
In general, the upper-bound of the approximation error is determined by the magnitude of \textit{order-reversing entries} in $\mathbf{T}^*_{\mathrm{FGW}}$, i.e., pairs of $(I,J)$ and $(I',J')$ such that $(I-I')(J-J')<0$.
Therefore, the error is expected to remain small as the FGW term in Eq.~\eqref{eq:fgw_state} encourages alignment between tokens of similar normalized positions, thus
producing a \textit{sparse} and \textit{block-diagonal} transport plan $\mathbf{T}_{\mathrm{FGW}}^*$ with few order-reversing entries.
\begin{theorem}[Complexity of \algname]
    \label{thm:complexity}
    The time complexity of \algname\ is $\mathcal{O}\!\left(B K L L_b (d+S)\right)$, and the space complexity of \algname\ is $\mathcal{O}\!\left(K L L_b + K L d + L_bd\right)$,
    where $K$ is the number of sampled rollouts of length $L$, $L_b$ is the barycenter size, $d$ is the hidden state dimension.
    $S$ and $B$ are the number of Sinkhorn and barycenter optimization iterations for solving Eq.~\eqref{eq:fgw_linear}, respectively.
    ~\footnote{
        Without loss of generality,
        we assume $\mathcal{O}(K^{(+)})\approx\mathcal{O}(K^{(-)})\approx\mathcal{O}(K)$
        and $\mathcal{O}(L^{(c)}_k)\approx\mathcal{O}(L)$ for all $k\in[1,K]$ and $c\in\{+,-\}$.
    }.
\end{theorem}

The detailed proof of Theorem~\ref{thm:complexity} can be found in Appendix~\ref{app:proof_complexity}.
Notably, \algname\ scales \textit{linearly} w.r.t. the rollout length $L$ in both time and space.
With a fixed and relatively small barycenter size $L_b=32$,
\algname\ enables efficient and scalable credit assignment for group-based RLVR methods across various reasoning tasks, as shown in Section~\ref{subsec:exp_efficiency}.

\section{Experiments}\label{sec:exp}

In this section, we carry out extensive experiments and analysis to evaluate the proposed \algname\ from the following aspects:
\begin{itemize}[nosep, wide=0pt, leftmargin=*, after=\strut, label=\textbullet]
    \item \textbf{Q1:} How effective is \algname\ compared to existing credit assignment approaches?
    \item \textbf{Q2:} How does \algname\ affect the training dynamics of group-based RLVR methods?
    \item \textbf{Q3:} How does \algname\ balance between effectiveness and efficiency?
    \item \textbf{Q4:} How good is the FGW distance as a novelty metrics for \algname?
    \item \textbf{Q5:} How sensitive is \algname\ to different hyperparameter settings and design choices?
\end{itemize}

\subsection{Experimental Setup}\label{subsec:exp_setup}
For \textbf{datasets}, we train group-based RLVR methods on the MATH dataset~\citep{hendrycks2021measuring},
and evaluate \algname\ across five different math reasoning benchmarks, including MATH500, AMC23, AIME24, AIME25~\citep{lightman2024let}, and Minerva Math~\citep{lewkowycz2022solving};
for \textbf{models}, we benchmark across three models, including Qwen3-1.7B-Base, Qwen3-4B-Base, and Qwen3-8B-Base~\citep{yang2025qwen3};
for \textbf{baselines}, we compare \algname\ with three existing credit assignment methods,
including GRPO-$\lambda$~\citep{parthasarathi2025grpo}, Entropy Adv.~\citep{cheng2026reasoning}, and SHEAR~\citep{chen2026hidden},
when applied to three different group-based RLVR methods, including GRPO~\citep{shao2024deepseekmath}, Dr.GRPO~\citep{liu2025understanding}, and DAPO~\citep{yu2026dapo}.
GRPO is adopted as the default RLVR baseline in all experiments unless noted otherwise;
for \textbf{evaluation metrics}, we adopt the average accuracy for MATH500 and Minerva Math, and the Pass@16 for AMC23, AIME24, and AIME25.
More detailed experimental setup are provided in Appendix~\ref{app:exp_add_setup}.

\subsection{Main Results}\label{subsec:exp_main}
\begin{table*}[t]
    \centering
    \caption{Benchmark results of credit assignment methods on GRPO~\citep{shao2024deepseekmath}.
        The 1st/2nd best results are highlighted in \textbf{bold}/\underline{underline}, respectively.
        Our methods are highlighted in \colorbox{red!8}{\!red box\!}.
    }
    \label{tab:bench_grpo}
    \vspace{-8pt}
    \begingroup
    \newcommand{\benchstd}[1]{\,{\scriptsize\textcolor{gray}{$\pm$#1}}}
    \setlength{\tabcolsep}{8pt}
    \renewcommand{\arraystretch}{1.15}
    \resizebox{\linewidth}{!}{%
        \begin{tabular}{l|ccccc|c}
            \noalign{\hrule height 1pt}
            \textbf{Method}    & \textbf{MATH500}               & \textbf{Minerva Math}          & \textbf{AMC23}                 & \textbf{AIME24}                & \textbf{AIME25}                & \textbf{Average} \\
            \hline
            \multicolumn{7}{c}{\rule{0pt}{2.5ex}\textbf{\textit{Qwen3-1.7B-Base}}}                                                                                                                                     \\
            \hline
            GRPO               & 64.8\benchstd{2.8}             & 18.2\benchstd{1.7}             & 72.5\benchstd{2.2}             & 14.0\benchstd{1.3}             & 12.7\benchstd{2.4}             & 36.4             \\
            ~~+~GRPO-$\lambda$ & 65.7\benchstd{1.2}             & 17.6\benchstd{1.5}             & 72.0\benchstd{1.0}             & \underline{16.0}\benchstd{0.6} & 14.0\benchstd{1.2}             & 37.1             \\
            ~~+~Entropy Adv.   & 66.5\benchstd{1.8}             & 19.1\benchstd{1.4}             & 74.0\benchstd{1.8}             & 15.3\benchstd{1.2}             & 13.3\benchstd{2.1}             & 37.6             \\
            ~~+~SHEAR          & 65.9\benchstd{1.3}             & 18.7\benchstd{0.9}             & 74.5\benchstd{0.4}             & 12.7\benchstd{0.9}             & 13.3\benchstd{1.5}             & 37.0             \\
            \rowcolor{red!8}
            ~~+~\algname-seq   & \textbf{67.3}\benchstd{2.0}    & \underline{20.3}\benchstd{1.2} & \underline{76.0}\benchstd{1.4} & \textbf{18.7}\benchstd{0.6}    & \underline{16.0}\benchstd{1.3} & 39.7             \\
            \rowcolor{red!8}
            ~~+~\algname-tok   & \underline{67.1}\benchstd{2.1} & \textbf{21.7}\benchstd{1.3}    & \textbf{77.0}\benchstd{1.6}    & \textbf{18.7}\benchstd{0.6}    & \textbf{16.7}\benchstd{1.2}    & \textbf{40.2}    \\
            \hline
            \multicolumn{7}{c}{\rule{0pt}{2.5ex}\textbf{\textit{Qwen3-4B-Base}}}                                                                                                                                       \\
            \hline
            GRPO               & 78.1\benchstd{1.5}             & 31.9\benchstd{1.8}             & 84.5\benchstd{1.7}             & 33.3\benchstd{1.9}             & 29.3\benchstd{2.1}             & 51.4             \\
            ~~+~GRPO-$\lambda$ & 76.9\benchstd{0.3}             & 32.1\benchstd{0.6}             & 86.5\benchstd{1.1}             & 32.7\benchstd{0.6}             & 30.0\benchstd{0.6}             & 51.6             \\
            ~~+~Entropy Adv.   & 77.5\benchstd{1.6}             & 33.0\benchstd{1.1}             & 86.0\benchstd{1.9}             & 34.7\benchstd{1.3}             & 31.3\benchstd{1.3}             & 52.5             \\
            ~~+~SHEAR          & 78.5\benchstd{0.8}             & 32.8\benchstd{1.3}             & 85.0\benchstd{0.9}             & 33.3\benchstd{1.5}             & 30.7\benchstd{1.2}             & 52.1             \\
            \rowcolor{red!8}
            ~~+~\algname-seq   & \underline{80.9}\benchstd{1.7} & \textbf{34.7}\benchstd{2.3}    & \underline{87.5}\benchstd{1.1} & \underline{36.7}\benchstd{1.5} & \underline{32.7}\benchstd{0.9} & 54.5             \\
            \rowcolor{red!8}
            ~~+~\algname-tok   & \textbf{82.1}\benchstd{1.1}    & \underline{34.1}\benchstd{1.5} & \textbf{88.0}\benchstd{1.3}    & \textbf{37.3}\benchstd{1.2}    & \textbf{34.0}\benchstd{1.5}    & \textbf{55.1}    \\
            \hline
            \multicolumn{7}{c}{\rule{0pt}{2.5ex}\textbf{\textit{Qwen3-8B-Base}}}                                                                                                                                       \\
            \hline
            GRPO               & 81.8\benchstd{1.3}             & 32.2\benchstd{0.9}             & 89.5\benchstd{2.0}             & 52.0\benchstd{1.3}             & 40.0\benchstd{1.7}             & 59.1             \\
            ~~+~GRPO-$\lambda$ & 82.1\benchstd{0.1}             & 32.4\benchstd{0.5}             & 90.5\benchstd{0.4}             & 53.3\benchstd{0.6}             & 41.3\benchstd{0.6}             & 59.9             \\
            ~~+~Entropy Adv.   & 83.9\benchstd{0.4}             & 33.8\benchstd{1.3}             & 91.0\benchstd{1.9}             & 52.7\benchstd{1.7}             & 42.7\benchstd{1.3}             & 60.8             \\
            ~~+~SHEAR          & 82.9\benchstd{0.7}             & 32.0\benchstd{1.9}             & 90.0\benchstd{2.0}             & 54.7\benchstd{1.3}             & 42.0\benchstd{1.2}             & 60.3             \\
            \rowcolor{red!8}
            ~~+~\algname-seq   & \underline{84.7}\benchstd{0.3} & \underline{34.6}\benchstd{1.1} & \underline{92.5}\benchstd{2.4} & \underline{58.0}\benchstd{1.2} & \underline{43.3}\benchstd{1.9} & 62.6             \\
            \rowcolor{red!8}
            ~~+~\algname-tok   & \textbf{85.2}\benchstd{0.5}    & \textbf{35.0}\benchstd{1.6}    & \textbf{94.0}\benchstd{1.9}    & \textbf{58.7}\benchstd{1.2}    & \textbf{46.7}\benchstd{1.7}    & \textbf{63.9}    \\
            \noalign{\hrule height 1pt}
        \end{tabular}%
    }
    \endgroup
\end{table*}

\paragraph{Benchmark Results.} The benchmark results of \algname\ and existing credit assignment methods on GRPO~\citep{shao2024deepseekmath} are shown in Table~\ref{tab:bench_grpo}. We observe that
\textbf{1) \algname\ consistently achieves the best performance across all datasets and LLMs.}
Specifically, \algname\ consistently improves GRPO by up to 4.8\% in the average performance, while also outperforming all
credit assignment baselines with up to 4.0\% improvement.
\textbf{2) Local token signals alone can be misleading for credit assignment.}
While all credit assignment baselines, on average, improves the original GRPO,
they can sometimes degrade the performance of GRPO in individual cases, e.g., GRPO-$\lambda$ on Qwen3-4B-Base for AIME24, and Entropy Adv. on Qwen3-1.7B-Base for MATH500,
showing that relying solely on local structural or uncertainty measure of tokens for credit assignment may be insufficient or even harmful.
\textbf{3) Effective credit assignment at token level leads to superior performance.}
We can see that \algname-tok outperforms \algname-seq on most datasets and LLMs, demonstrating that our method effectively captures novel reasoning behaviors at token level,
thus leading to improved reasoning performance.
Additional results on Dr.GRPO and DAPO are provided in Appendix~\ref{app:exp_add_bench}, which further demonstrates the effectiveness of \algname.

\paragraph{Training Dynamics.}
\begin{figure*}[t]
    \vspace{-5pt}
    \centering
    \includegraphics[width=\linewidth, trim=0 5 5 5, clip]{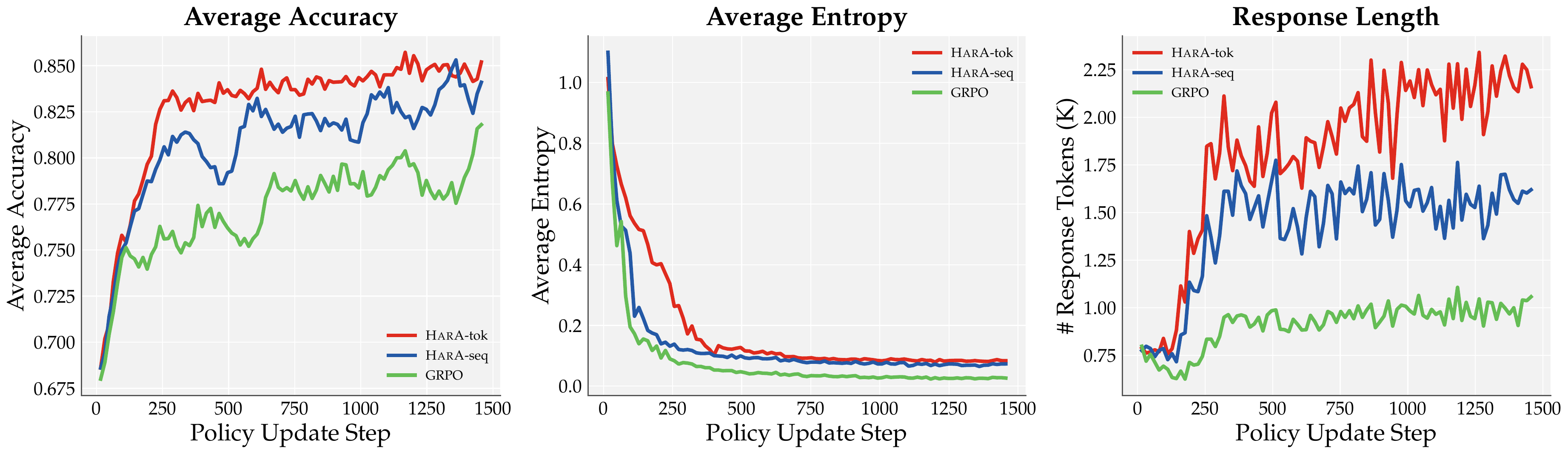}
    \vspace{-15pt}
    \caption{Training dynamics of \algname\ and GRPO on MATH500 using Qwen3-8B-Base.}
    \vspace{-15pt}
    \label{fig:exp_dynamics}
\end{figure*}

To further understand how \algname\ affects the training of group-based RLVR methods,
we visualize the training dynamics of GRPO and \algname\ on MATH500 using Qwen3-8B-Base in Figure~\ref{fig:exp_dynamics},
tracking three key metrics: \textit{accuracy}, \textit{response length}, and \textit{entropy} of the trained LLM.
We can oberve that \textbf{1) \algname\ enables faster convergence and improved final performance.} \algname\ identifies novel reasoning rollouts and tokens to emphasize during training,
thus converges faster and achieve sa better final performance than baseline GRPO.
\textbf{2) \algname\ leads to more extensive exploration.} We can see that \algname\ maintains higher entropy than baseline GRPO during training, thus preserving more exploration capability of the trained model.
Moreover, \algname\ generates longer responses than baseline GRPO, indicating that it encourages the model to explore more diverse reasoning paths by  identifying and emphasizing on novel reasoning behaviors.

\subsection{Efficiency Results}\label{subsec:exp_efficiency}
\begin{figure*}[t]
    \centering
    \begin{minipage}[t]{0.3\textwidth}
        \vspace{0pt}
        \centering
        \includegraphics[width=\linewidth, trim=5 5 5 5, clip]{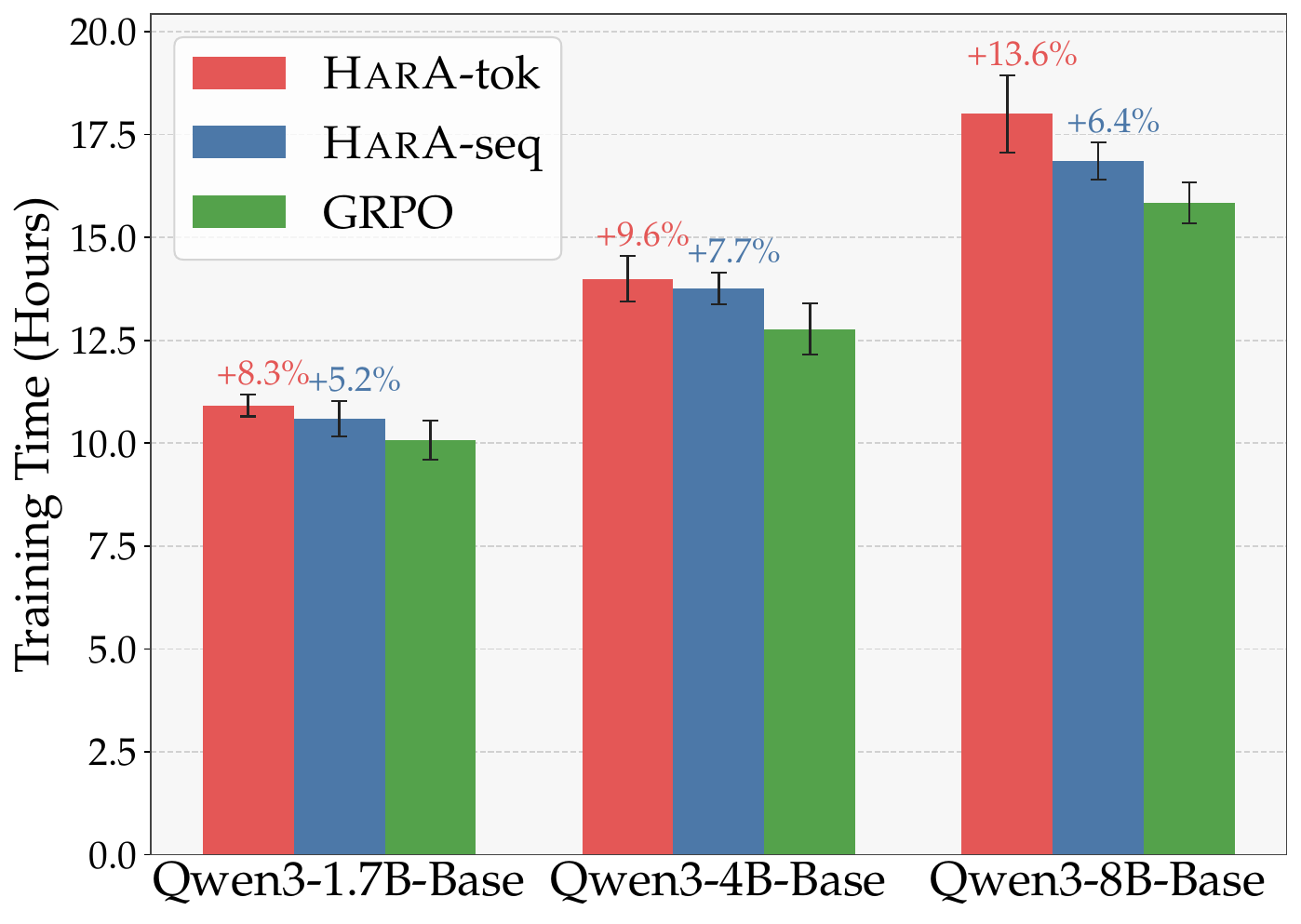}
        \vspace{-18pt}
        \captionof{figure}{Training overhead of \algname\ on MATH.\label{fig:train_time_scale}}
    \end{minipage}
    \hfill
    \begin{minipage}[t]{0.3\textwidth}
        \vspace{0pt}
        \centering
        \includegraphics[width=\linewidth, trim=5 5 5 5, clip]{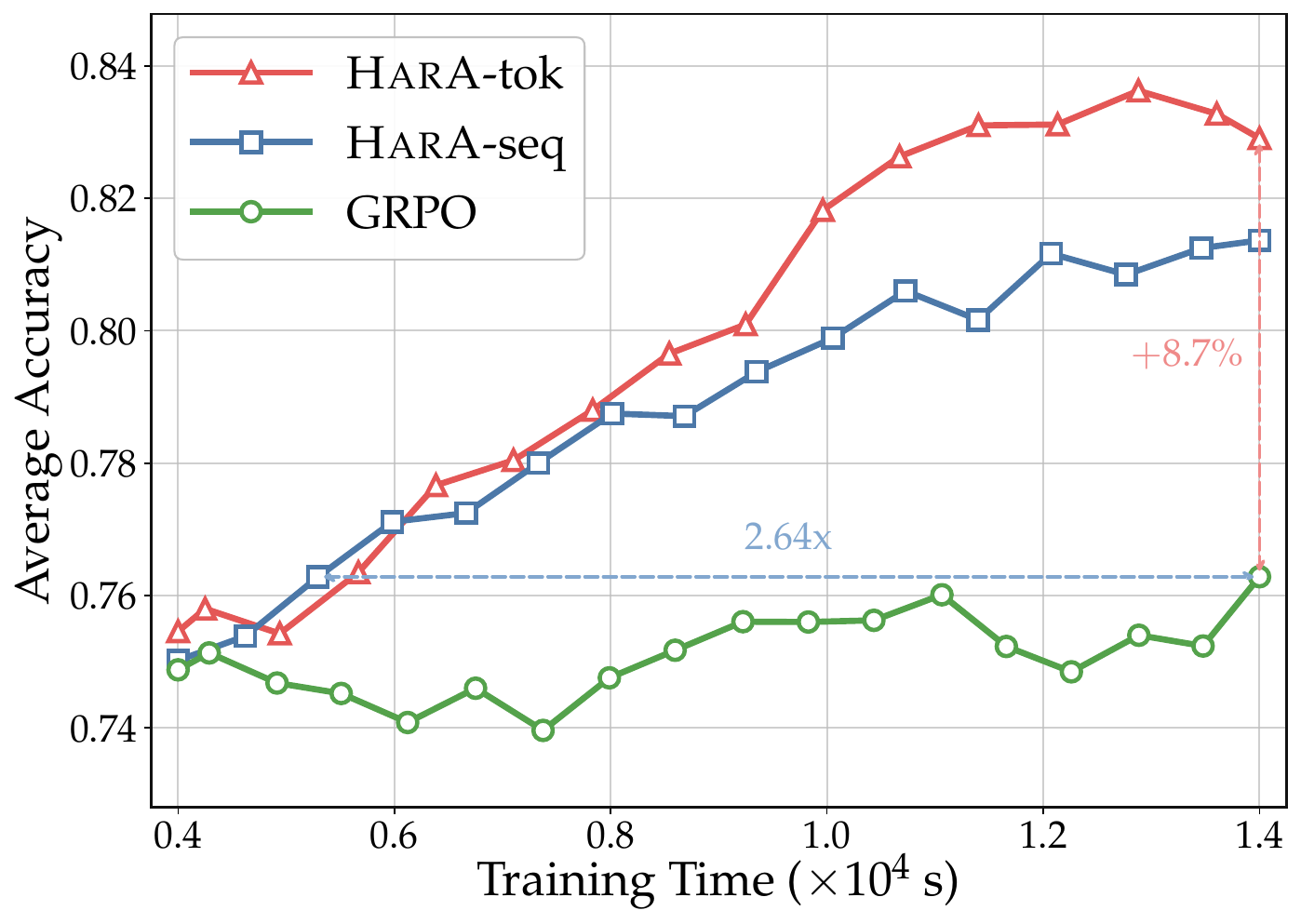}
        \vspace{-18pt}
        \captionof{figure}{Average accuracy vs. training time on MATH.}
        \label{fig:acc_vs_time}
    \end{minipage}
    \hfill
    \begin{minipage}[t]{0.36\textwidth}
        \vspace{0pt}
        \centering
        \captionof{table}{Comparison of variants of \algname\ on Qwen3-8B-Base.}
        \label{tab:exp_variants}
        \vspace{-10pt}
        \renewcommand{\arraystretch}{1.15}
        \resizebox{\linewidth}{!}{%
            \begin{tabular}{l|cc}
                \noalign{\hrule height 1pt}
                \textbf{Method}      & \textbf{MATH500} & \textbf{AMC23} \\
                \hline
                GRPO                 & 81.8             & 89.5           \\
                ~~+~\algname(W)-seq  & 82.3             & 90.0           \\
                ~~+~\algname(GW)-seq & 80.6             & 87.0           \\
                ~~+~\algname(H)-seq  & 81.2             & 86.5           \\
                \rowcolor{red!8}
                ~~+~\algname-seq     & \textbf{84.7}    & \textbf{92.5}  \\
                \noalign{\hrule height 1pt}
                ~~+~\algname(W)-tok  & 83.0             & 90.5           \\
                ~~+~\algname(GW)-tok & 79.7             & 87.0           \\
                \rowcolor{red!8}
                ~~+~\algname-tok     & \textbf{85.2}    & \textbf{94.0}  \\
                \noalign{\hrule height 1pt}
            \end{tabular}%
        }
    \end{minipage}
    \vspace{-15pt}
\end{figure*}

As shown in Figure~\ref{fig:train_time_scale}, \algname\ introduces a mild additional overhead of up to 13.6\% in total training time compared to baseline GRPO,
making it practical to be applied to standard RLVR training.
Although \algname\ does introduce additional overhead, Figure~\ref{fig:acc_vs_time} shows that, compared with GRPO, \algname\ reaches the \textbf{same accuracy with up to 2.64$\times$ speed-up} in training time,
and achieves an \textbf{accuracy improvement of 8.7\% using the same training time}. This indicates that \algname\ maintains a good balance between effectiveness and efficiency for improving group-based RLVR methods.

\subsection{Further Studies}\label{subsec:exp_further}
\paragraph{Comparision of Novelty Metrics.}
To verify the effectiveness of the FGW formulation in \algname, we compare the performance of different variants of \algname\ using different novelty metrics for $\omega$ in Eq.~\eqref{eq:alg-obj},
including Wasserstein (W) and Gromov-Wasserstein (GW) distances in Eq.~\ref{eq:fgw_dist}, and an additional metrics defined as the embedding distance between the barycenter and the last hidden state of each rollout (H),
i.e., $\omega_{k,i}=\|\mathbf{X}_{k,L_k}^{(c)}-\mathbf{X}_{b'}^{(c)}\|_2^2$ where $\mathbf{X}_{b'}^{(c)}=\frac{1}{K}\sum_{k'=1}^K\mathbf{X}_{k',L_k}^{(c)}$.
The results are included in Table~\ref{tab:exp_variants}, showing that \textbf{1) the FGW distance adopted by \algname\ is the most effective novelty metric for credit assignment} with consistent outperformance over all other metrics.
\textbf{2) The GW distance, while effective in aligning tokens in Figure~\ref{fig:obs}, cannot capture semantic novelty of a reasoning behavior}, resulting in degraded performance even compared to baseline GRPO.

\paragraph{Additional Studies.} Due to space constraints, we present additional studies of \algname,
including \textbf{parameters studies} on the FGW trade-off parameter $\alpha$, and an \textbf{ablation study on the anchor-guided linearization}, in Appendix~\ref{app:exp_add_study}.
\section{Related Work}\label{sec:related}

\paragraph{Fused Gromov-Wasserstein Barycenter.}
Fused Gromov-Wasserstein (FGW) barycenter has been predominantly used in graph learning and time-series modeling. Specifically,
\citep{brogat2022learning} leverage an input-dependent FGW barycenters for interpolatation within the space of attributed graphs;
\citep{cheng2021dynamical} use dynamically weighted Wasserstein barycenters to model gradual transitions in multivariate time series.
\citep{thual2022aligning} introduce fused unbalanced Gromov-Wasserstein transport to align individual cortical surfaces;
\citep{zeng2024hierarchical} leverage an FGW barycenter for cluster-level multi-network alignment.

\paragraph{Credit Assignment for RLVR.} Existing credit assignment methods for RLVR operate on sequence, span, and token levels.
For sequence-level methods, \citep{liu2026leveraging} measures intra-group error diversity to reshape the advantage signal of incorrect rollouts.
For span-level methods, \citep{chen2026hidden} emphasizes on spans that are more distant from the opposing group based on Wasserstein distance,
For token-level methods, \citep{parthasarathi2025grpo} assigns higher credits to nearby tokens and propagates such credits to preceding tokens with exponential decay,
while \citep{cheng2026reasoning} and \citep{xie2026unlocking} emphasize on tokens with higher entropy-based uncertainty.
While promising, none of them identifies semantically novel reasoning behaviors at the finest token level.
\section{Conclusion}\label{sec:con}
In this paper, we propose \algname, a novel hierarchical credit assignment method for group-based RLVR methods.
\algname\ represents each sampled rollout as a distribution over the hidden states and locations of tokens,
and leverages the Fused Gromov-Wasserstein (FGW) barycenters of all rollouts of the same outcome to capture the internal reasoning patterns in the latent space.
By measuring the semantic novelty of a reasoning element based on its contribution to the FGW distance to the barycenter,
\algname\ reweighs token-level advantage of group-based RLVR methods to encourage novel reasoning behaviors at flexible granularities.
Extensive experiments demonstrate that our plug-and-play method significantly improves the performance of diverse group-based RLVR methods while maintaining a good balance between effectiveness and efficiency.
\section*{AI use statement}

In this work, generative AI tools are used solely for polishing the grammar and readability of our work.
We have not used generative AI tools for other tasks with required or recommended disclosure.
We have reviewed all AI-assisted work, and take full responsibility for the final content of this work,
including text, claims and artifacts.

\bibliography{sections/reference}
\bibliographystyle{iclr2027_conference}

\appendix
\newpage
\section{Proof}\label{app:proof}
\subsection{Proof of Lemma~\ref{lem:fgw_linear}}\label{app:proof_fgw_linearization}

\begin{lemma*}[Anchor-Guided Linearization of FGW Distance]
    $\mathrm{FGW}_{2,\alpha}$ in Eq.~\eqref{eq:fgw_state} can be linearized into the following Wasserstein distance $\mathrm{FGW}_{2,\alpha}^{\text{linear}}$
    by fixing the second transport plan $\mathbf{T}^{(c)}_{\text{anchor}}=\frac{1}{2}\left(\delta_{(1,1)}+\delta_{(l_k,l_b)}\right)$ in its $\mathrm{GW}_2$ term that pre-aligns reasoning anchors $y^{(c)}_{k,1}$ and $y^{(c)}_{k,l_k}$
    \begin{equation}
        \!\!\!
        \mathrm{FGW}_{2,\alpha}^{\mathrm{linear}}\!
        =\!
        \min_{\mathbf{T}\in\Pi(\bs{\mu}_k,\bs{\mu}_b)}
        \mathbb{E}_{\left(y^{(c)}_{k,i},y^{(c)}_{b,j}\right)\sim\mathbf{T}}
        \left[
            (1-\alpha)g^2\left(y^{(c)}_{k,i},y^{(c)}_{b,j}\right)
            +
            2\alpha
            \left|
            \frac{i-1}{l_k-1}
            -
            \frac{j-1}{l_b-1}
            \right|^2
            \right]\!
    \end{equation}
\end{lemma*}

\begin{proof}
    Based on Eq.~\eqref{eq:fgw_dist}, the $\mathrm{FGW}_{2,\alpha}$ objective between the graph representation rollout
    $\G^{(c)}_k$ of rollout $y^{(c)}_k$ and the barycenter $\G_b$ can be expanded as follows
    \begin{align}\label{eq:expand_fgw}
        \mathrm{FGW}_{2,\alpha}
        =
        \min_{\mathbf{T}\in\Pi(\mu_k,\mu_b)}
        \Bigg\{
         & (1-\alpha)
        \mathbb{E}_{(y_{k,i}^{(c)},y_{b,j}^{(c)})\sim\mathbf{T}}
        \left[
            g^2\!\left(y_{k,i}^{(c)},y_{b,j}^{(c)}\right)
            \right]
        \nonumber     \\
         & +
        \alpha
        \underbrace{
            \mathbb{E}_{\substack{
                    (y_{k,i}^{(c)},y_{b,j}^{(c)})\sim\mathbf{T}\\
                    (y_{k,i'}^{(c)},y_{b,j'}^{(c)})\sim\mathbf{T}
                }}
            \left[
                \left|
                f\!\left(y_{k,i}^{(c)},y_{k,i'}^{(c)}\right)
                -
                f\!\left(y_{b,j}^{(c)},y_{b,j'}^{(c)}\right)
                \right|^2
                \right]
        }_{\mathrm{GW}_2}
        \Bigg\}.
    \end{align}

    $\mathrm{FGW}_{2,\alpha}$ is difficult to solve as $\mathrm{GW}_2$ is non-convex and quadratic w.r.t $\mathbf{T}$.
    To improve efficiency, we linearize $\mathrm{GW}_2$ by fixing its second transport plan to $\mathbf{T}_{\mathrm{anchor}}^{(c)}=\frac{1}{2}\left(\delta_{(1,1)}+\delta_{(l_k,l_b)}\right)$
    that pre-aligns the reasoning anchors by assigning mass only to the two entries $(1,1)$ and $(l_k,l_b)$ that aligns $(y^{(c)}_{k,1}, y^{(c)}_{b,1})$ and $(y^{(c)}_{k,l_k},y^{(c)}_{b,l_b})$,
    In this way, the linearized GW term $\mathrm{GW}_2^{\mathrm{linear}}$ can be derived as follows
    \begin{align}\label{eq:derive_linear_gw}
        \mathrm{GW}_2^{\mathrm{linear}}
        \!\! & =\!2\cdot\mathbb{E}_{\substack{
                                (y_{k,i}^{(c)},y_{b,j}^{(c)})\sim\mathbf{T}\\
                                (y_{k,i'}^{(c)},y_{b,j'}^{(c)})\sim\mathbf{T}_{\mathrm{anchor}}^{(c)}
                            }}
        \left[
            \left|
            f\!\left(y_{k,i}^{(c)},y_{k,i'}^{(c)}\right)
            -
            f\!\left(y_{b,j}^{(c)},y_{b,j'}^{(c)}\right)
            \right|^2
            \right]
        \nonumber                                                                                    \\
             & =\!
        \mathbb{E}_{(y_{k,i}^{(c)},y_{b,j}^{(c)})\sim\mathbf{T}}
        \left[
            \sum_{(i',j')\in\{(1,1),(l_k,l_b)\}}
            \left|
            f\!\left(y_{k,i}^{(c)},y_{k,i'}^{(c)}\right)
            -
            f\!\left(y_{b,j}^{(c)},y_{b,j'}^{(c)}\right)
            \right|^2
            \right]
        \nonumber                                                                                    \\
             & =\!
        \mathbb{E}_{(y_{k,i}^{(c)},y_{b,j}^{(c)})\sim\mathbf{T}}
        \left[
            \left|
            f\!\left(y_{k,i}^{(c)},y_{k,1}^{(c)}\right)
            -
            f\!\left(y_{b,j}^{(c)},y_{b,1}^{(c)}\right)
            \right|^2
            +
            \left|
            f\!\left(y_{k,i}^{(c)},y_{k,l_k}^{(c)}\right)
            -
            f\!\left(y_{b,j}^{(c)},y_{b,l_b}^{(c)}\right)
            \right|^2
            \right]
        \nonumber                                                                                    \\
             & =\!
        \mathbb{E}_{(y_{k,i}^{(c)},y_{b,j}^{(c)})\sim\mathbf{T}}
        \left[
            \left|
            \frac{i-1}{l_k-1}
            -
            \frac{j-1}{l_b-1}
            \right|^2
            +
            \left|
            \frac{l_k-i}{l_k-1}
            -
            \frac{l_b-j}{l_b-1}
            \right|^2
            \right]
        \nonumber                                                                                    \\
             & =\!
        \mathbb{E}_{(y_{k,i}^{(c)},y_{b,j}^{(c)})\sim\mathbf{T}}
        \left[
            \left|
            \frac{i-1}{l_k-1}
            -
            \frac{j-1}{l_b-1}
            \right|^2
            +
            \left|
            -\left(
            \frac{i-1}{l_k-1}
            -
            \frac{j-1}{l_b-1}
            \right)
            \right|^2
            \right]
        \nonumber                                                                                    \\
             & =\!
        2\,
        \mathbb{E}_{(y_{k,i}^{(c)},y_{b,j}^{(c)})\sim\mathbf{T}}
        \left[
            \left|
            \frac{i-1}{l_k-1}
            -
            \frac{j-1}{l_b-1}
            \right|^2
            \right].
    \end{align}

    Note that we apply a scaling factor of 2 here to account for both the anchor points $y^{(c)}_{k,1}$ and $y^{(c)}_{k,l_k}$ in the linearization of $\mathrm{GW}_2$.
    Substituting $\mathrm{GW}_2$ in Eq.~\eqref{eq:expand_fgw} by $\mathrm{GW}_2^{\mathrm{linear}}$ in Eq.~\eqref{eq:derive_linear_gw} yields
    \begin{equation}
        \!\!
        \mathrm{FGW}_{2,\alpha}^{\mathrm{linear}}
        =
        \min_{\mathbf{T}\in\Pi(\mu_k,\mu_b)}
        \mathbb{E}_{(y_{k,i}^{(c)},y_{b,j}^{(c)})\sim\mathbf{T}}
        \left[
            (1-\alpha)
            g^2\!\left(y_{k,i}^{(c)},y_{b,j}^{(c)}\right)
            +
            2\alpha
            \left|
            \frac{i-1}{l_k-1}
            -
            \frac{j-1}{l_b-1}
            \right|^2
            \right]
    \end{equation}
    which give a Wasserstein distance formulation that is linear w.r.t $\mathbf{T}$, thus proving the Lemma.
\end{proof}

\subsection{Proof of Theorem~\ref{thm:linearization_error}}\label{app:proof_linearization_error}
\begin{theorem*}[Approximation Error of Anchor-Guided Linearization]
    Let $\mathbf{T}_{\mathrm{FGW}}^*\in\Pi(\mu_k,\mu_b)$ be an optimal transport plan that minimizes the original FGW problem with uniform marginal distributions $\mu_k$ and $\mu_b$.
    The approximation error of the anchor-guided linearization in Lemma~\ref{lem:fgw_linear} is bounded as follows
    \begin{equation}
        0\leq
        \mathrm{FGW}_{2,\alpha}^{\text{linear}}(\mathcal{G}_k^{(c)},\mathcal{G}_b^{(c)})
        -
        \mathrm{FGW}_{2,\alpha}(\mathcal{G}_k^{(c)},\mathcal{G}_b^{(c)})
        \leq
        4\alpha\,\mathcal{C}(\mathbf{T}_{\mathrm{FGW}}^*),
    \end{equation}
    where
    \begin{equation}
        \mathcal{C}(\mathbf{T})
        =
        \mathbb{E}_{\mathbf{T}\otimes\mathbf{T}}
        \left[
            \frac{|I-I'|}{L_k-1}
            \frac{|J-J'|}{L_b-1}
            \mathbbm{1}\!\left\{(I-I')(J-J')<0\right\}
            \right],
        \qquad
        (I,J),(I',J')\overset{\mathrm{i.i.d.}}{\sim}\mathbf{T}.
    \end{equation}
\end{theorem*}

\begin{proof}
    Let $F(\mathbf{T})$ and $F_{\mathrm{lin}}(\mathbf{T})$ denote the original FGW formulation and its anchor-guided linearization evaluated at a common transport plan $\mathbf{T}\in\Pi(\mu_k,\mu_b)$, i.e.,
    \begin{equation}\label{eq:dist_common_T}
        \begin{aligned}
            F(\mathbf{T})                & =(1-\alpha)\mathrm{W}_2(\mathbf{T},\mathbf{M}_{k,b})+\alpha\,\mathrm{GW}_2(\mathbf{T},\mathbf{C}_k,\mathbf{C}_b),                                                \\
            F_{\mathrm{lin}}(\mathbf{T}) & =(1-\alpha)\mathrm{W}_2(\mathbf{T},\mathbf{M}_{k,b})+2\alpha\,\mathbb \sum_{i=1,j=1}^{L_k,L_b}\left(\frac{i-1}{L_k-1}-\frac{j-1}{L_b-1}\right)^2\mathbf{T}(i,j).
        \end{aligned}
    \end{equation}
    Based on $(I,J),(I',J')$ sampled i.i.d. from $\mathbf{T}$, define the normalized token locations $P$, $Q$, $P'$, and $Q'$ as follows
    \[
        P = \frac{I-1}{L_k-1}, \quad Q = \frac{J-1}{L_b-1}, \quad P' = \frac{I'-1}{L_k-1}, \quad Q' = \frac{J'-1}{L_b-1}.
    \]

    In this way, we can rewrite $\mathrm{GW}_2(\mathbf{T},\mathbf{C}_k,\mathbf{C}_b)$ and its linearization in Eq.~\eqref{eq:dist_common_T} as follows
    \begin{align}\label{eq:gw_exp}
        \mathrm{GW}_2(\mathbf T,\mathbf C_k,\mathbf C_b)
         & =
        \sum_{i,i'=1}^{L_k}
        \sum_{j,j'=1}^{L_b}
        \left(
        \frac{|i-i'|}{L_k-1}
        -
        \frac{|j-j'|}{L_b-1}
        \right)^2
        \mathbf T(i,j)\mathbf T(i',j')
        \nonumber \\
         & =
        \mathbb E_{\mathbf T\otimes\mathbf T}
        \left[
            \left(
            |P-P'|-|Q-Q'|
            \right)^2
            \right]
    \end{align}
    \begin{equation}\label{eq:linear_exp}
        \sum_{i=1,j=1}^{L_k,L_b}\left(\frac{i-1}{L_k-1}-\frac{j-1}{L_b-1}\right)^2\mathbf{T}(i,j)=
        \mathbb E_{\mathbf T}\left[(P-Q)^2\right].
    \end{equation}

    For any $a,b\in\mathbb R$, we have
    \begin{equation}
        (|a|-|b|)^2
        =
        (a-b)^2
        -
        4|ab|\mathbbm 1\{ab<0\}.
        \label{eq:abs_identity}
    \end{equation}
    Applying Eq.~\eqref{eq:abs_identity} with
    $a=P-P'$ and $b=Q-Q'$ gives
    \begin{align}
         & \mathbb E_{\mathbf T\otimes\mathbf T}
        \left[
            \left(|P-P'|-|Q-Q'|\right)^2
            \right]
        \nonumber                                \\
         & =
        \mathbb E_{\mathbf T\otimes\mathbf T}
        \left[
            \left((P-Q)-(P'-Q')\right)^2
            \right]
        -
        4\mathcal C(\mathbf T).
        \label{eq:gw_decomp}
    \end{align}
    where
    \begin{equation}\label{eq:error_term}
        \mathcal{C}(\mathbf{T})
        =
        \mathbb{E}_{\mathbf{T}\otimes\mathbf{T}}
        \left[
            |P-P'||Q-Q'|\mathbbm{1}\{(P-P')(Q-Q')<0\}
            \right].
    \end{equation}

    Since both $(P,Q)$ and $(P',Q')$ are i.i.d., we have
    \begin{align}
         & \mathbb{E}_{\mathbf{T}\otimes\mathbf{T}}
        \left[
            \left((P-Q)-(P'-Q')\right)^2
            \right]
        \nonumber                                   \\
         & \qquad=
        \mathbb{E}_{\mathbf{T}}\left[(P-Q)^2\right]
        +
        \mathbb{E}_{\mathbf{T}}\left[(P'-Q')^2\right]
        -
        2\left(\mathbb{E}_{\mathbf{T}}[P-Q]\right)^2
        \nonumber                                   \\
         & \qquad=
        2\mathbb{E}_{\mathbf{T}}\left[(P-Q)^2\right]
        -
        2\left(
        \mathbb{E}_{\mathbf{T}}[P]
        -
        \mathbb{E}_{\mathbf{T}}[Q]
        \right)^2.
        \label{eq:iid_identity}
    \end{align}
    Since $\mathbf{T}\in\Pi(\mu_k,\mu_b)$ where $\mu_k$ and $\mu_b$ are two uniform distributions, we have
    \begin{equation}
        \mathbb E[P]
        =
        \frac{1}{L_k}\sum_{i=1}^{L_k}\frac{i-1}{L_k-1}
        =
        \frac12,
        \qquad
        \mathbb E[Q]
        =
        \frac{1}{L_b}\sum_{j=1}^{L_b}\frac{j-1}{L_b-1}
        =
        \frac12.
    \end{equation}

    Therefore, we have
    \begin{equation}\label{eq:gw_exp_simplified}
        \mathbb{E}_{\mathbf{T}\otimes\mathbf{T}}
        \left[
            \left((P-Q)-(P'-Q')\right)^2
            \right]
        =
        2\mathbb{E}_{\mathbf{T}}\left[(P-Q)^2\right].
    \end{equation}

    Substituting Eq.~\eqref{eq:gw_exp_simplified} into Eq.~\eqref{eq:gw_decomp} and~\eqref{eq:gw_exp} yields
    \begin{align}\label{eq:gw_exp_sim_final}
        \mathrm{GW}_2(\mathbf T,\mathbf C_k,\mathbf C_b)
         & =\mathbb{E}_{\mathbf{T}\otimes\mathbf{T}}
        \left[
            \left(
            |P-P'|-|Q-Q'|
            \right)^2
            \right]
        \nonumber                                    \\
         & =
        2\mathbb{E}_{\mathbf{T}}\left[(P-Q)^2\right]
        -
        4\mathcal C(\mathbf T)
    \end{align}

    Comparing Eq.~\eqref{eq:gw_exp_sim_final} with Eq.~\eqref{eq:dist_common_T} and~\eqref{eq:linear_exp}
    yields that, for every $\mathbf{T}\in\Pi(\mu_k,\mu_b)$, we have
    \begin{equation}
        F_{\mathrm{lin}}(\mathbf T)-F(\mathbf T)
        =
        4\alpha\,\mathcal C(\mathbf T)
        \geq 0.
        \label{eq:pointwise_bound}
    \end{equation}

    Let
    \[
        \mathbf T_{\mathrm{FGW}}^*
        =\mathop{\arg\min}\limits_{\mathbf{T}\in\Pi(\mu_k,\mu_b)}F(\mathbf T),
        \qquad
        \mathbf T_{\mathrm{lin}}^*
        =\mathop{\arg\min}\limits_{\mathbf{T}\in\Pi(\mu_k,\mu_b)}F_{\mathrm{lin}}(\mathbf T)
    \]
    be optimal transport plans that minimize $F(\mathbf T)$ and $F_{\mathrm{lin}}(\mathbf T)$,
    we have
    \begin{equation}\label{eq:fgw_opt}
        \mathrm{FGW}_{2,\alpha}(\G_k^{(c)},\G_b^{(c)})=F(\mathbf T_{\mathrm{FGW}}^*),
        \qquad
        F(\mathbf T_{\mathrm{FGW}}^*)\leq F(\mathbf T_{\mathrm{lin}}^*)
    \end{equation}
    \begin{equation}\label{eq:fgw_linear_opt}
        \mathrm{FGW}_{2,\alpha}^{\mathrm{linear}}(\G_k^{(c)},\G_b^{(c)})=F_{\mathrm{lin}}(\mathbf T_{\mathrm{lin}}^*),
        \qquad
        F_{\mathrm{lin}}(\mathbf T_{\mathrm{lin}}^*)\leq F_{\mathrm{lin}}(\mathbf T_{\mathrm{FGW}}^*)
    \end{equation}

    For the lower bound, from Eq.~\eqref{eq:pointwise_bound} to~\eqref{eq:fgw_linear_opt}, we have
    \begin{equation}
        \mathrm{FGW}_{2,\alpha}^{\mathrm{linear}}(\G_k^{(c)},\G_b^{(c)})
        =
        F_{\mathrm{lin}}(\mathbf T_{\mathrm{lin}}^*)
        \geq
        F(\mathbf T_{\mathrm{lin}}^*)
        \geq
        F(\mathbf T_{\mathrm{FGW}}^*)
        =
        \mathrm{FGW}_{2,\alpha}(\G_k^{(c)},\G_b^{(c)})
    \end{equation}
    Therefore, we have
    \begin{equation}
        \mathrm{FGW}_{2,\alpha}^{\mathrm{linear}}(\G_k^{(c)},\G_b^{(c)})
        -
        \mathrm{FGW}_{2,\alpha}(\G_k^{(c)},\G_b^{(c)})
        \geq0.
        \label{eq:lower_error_bound}
    \end{equation}

    For the upper bound, from Eq.~\eqref{eq:pointwise_bound} to~\eqref{eq:fgw_linear_opt}, we have
    \begin{align}
        \mathrm{FGW}_{2,\alpha}^{\mathrm{linear}}(\G_k^{(c)},\G_b^{(c)})
         & =
        F_{\mathrm{lin}}(\mathbf T_{\mathrm{lin}}^*)
        \nonumber \\
         & \leq
        F_{\mathrm{lin}}(\mathbf T_{\mathrm{FGW}}^*)
        \nonumber \\
         & =
        F(\mathbf T_{\mathrm{FGW}}^*)
        +
        4\alpha\mathcal C(\mathbf T_{\mathrm{FGW}}^*),
        \nonumber \\
         & =
        \mathrm{FGW}_{2,\alpha}(\G_k^{(c)},\G_b^{(c)})
        +
        4\alpha\mathcal C(\mathbf T_{\mathrm{FGW}}^*).
    \end{align}
    Therefore, we have
    \begin{equation}
        \mathrm{FGW}_{2,\alpha}^{\mathrm{linear}}
        -
        \mathrm{FGW}_{2,\alpha}
        \leq
        4\alpha\mathcal C(\mathbf T_{\mathrm{FGW}}^*).
        \label{eq:upper_error_bound}
    \end{equation}
    Combining Eqs.~\eqref{eq:lower_error_bound} and~\eqref{eq:upper_error_bound} yields
    \begin{equation}
        0\leq
        =
        \mathrm{FGW}_{2,\alpha}^{\mathrm{linear}}(\mathcal G_k^{(c)},\mathcal G_b^{(c)})
        -
        \mathrm{FGW}_{2,\alpha}(\mathcal G_k^{(c)},\mathcal G_b^{(c)})
        \leq
        4\alpha\mathcal C(\mathbf T_{\mathrm{FGW}}^*).
    \end{equation}

    Finally, since
    \[
        P-P'=\frac{I-I'}{L_k-1},
        \qquad
        Q-Q'=\frac{J-J'}{L_b-1},
    \]
    replacing $P-P'$ and $Q-Q'$ in Eq.~\eqref{eq:error_term} yields
    \[
        \mathcal C(\mathbf T)
        =
        \mathbb E_{\mathbf T\otimes\mathbf T}
        \left[
            \frac{|I-I'|}{L_k-1}
            \frac{|J-J'|}{L_b-1}
            \mathbbm 1\!\left\{(I-I')(J-J')<0\right\}
            \right].
    \]
    In this way, we have proven Theorem~\ref{thm:linearization_error}.
\end{proof}
\subsection{Proof of Theorem~\ref{thm:complexity}}\label{app:proof_complexity}
\begin{theorem*}[Complexity of \algname]
    The time complexity of \algname\ is $\mathcal{O}\!\left(B K L L_b (d+S)\right)$,
    and the space complexity of \algname\ is $\mathcal{O}\!\left(K L L_b + K L d + L_bd\right)$,
    where $K$ is the number of sampled rollouts of length $L$, $L_b$ is the barycenter size, $d$ is the hidden state dimension.
    $S$ and $B$ are the number of Sinkhorn and barycenter optimization iterations for solving Eq.~\eqref{eq:fgw_linear}, respectively.
    ~\footnote{
        Without loss of generality,
        we assume $\mathcal{O}(K^{(+)})\approx\mathcal{O}(K^{(-)})\approx\mathcal{O}(K)$
        and $\mathcal{O}(L^{(c)}_k)\approx\mathcal{O}(L)$ for all $k\in[1,K]$ and $c\in\{+,-\}$.
    }.
\end{theorem*}

\begin{proof}
    \textbf{For time complexity}, the computation of \algname\ consists mainly of two alternating optimization steps~\citep{zeng2024hierarchical}:
    for each iteration, (1) solving Eq.~\eqref{eq:fgw_linear} for all $K$ Wasserstein OT problems,
    and (2) updating $\mathbf{X}_b^{(c)}$ based on the solved Wasserstein distance at the current iteration.
    For solving the Wasserstein distance for one rollout in Eq.~\eqref{eq:fgw_linear},
    constructing the cost matrix requires $\mathcal{O}(LL_b d)$ and $\mathcal{O}(LL_b)$ time for computing the pairwise distances of the $d$-dimensional hidden states and normalized locations, respectively.
    By entropy regularization, each Wasserstein OT problem can be solved in $\mathcal{O}(SLL_b)$ time via the Sinkhorn algorithm~\citep{gabriel2019computational}.
    Therefore, solving Eq.~\eqref{eq:fgw_linear} for all $K$ rollouts
    costs
    \[
        K\cdot\mathcal{O}\!\left(LL_b d+LL_b+SLL_b\right)=\mathcal{O}\!\left(KLL_b(d+S)\right)
    \]
    per iteration. Updating $\mathbf{X}_b^{(c)}$ requires $\mathcal{O}(KLL_b d)$ time.
    Therefore, over $B$ barycenter optimization iterations, the total time complexity is
    \[
        B\cdot\mathcal{O}\!\left(KLL_b(d+S)+KLL_b d\right)=\mathcal{O}\!\left(BKLL_b(d+S)\right).
    \]

    \textbf{For space complexity}, storing the transport plan matrices $\mathbf{T}_k^{(c)}\in\mathbb{R}^{L\times L_b}$
    for $K$ Wasserstein OT problems requires $\mathcal{O}(KLL_b)$ space.
    Storing the hidden states of $K$ rollouts and the barycenter requires $\mathcal{O}(KLd)$ and $\mathcal{O}(L_bd)$ space, respectively.
    Therefore, the total space complexity is
    \[
        \mathcal{O}\!\left(KLL_b+KLd+L_bd\right).
    \]
\end{proof}

\section{Detailed Experimental Setup}\label{app:exp_add_setup}
\subsection{Dataset Setup}
\paragraph{Training Dataset.} We train all models on the MATH dataset~\citep{hendrycks2021measuring} which contains 12,500 problems from mathematics competitions.
We use the publicly available version at \url{https://huggingface.co/datasets/DigitalLearningGmbH/MATH-lighteval}.
Following~\citep{lightman2024let}, we exclude problems inside MATH500 from the MATH dataset, and use the remaining 12,000 problems for training.

\paragraph{Evaluation Datasets.} We evaluate the trained models on five standard mathematical reasoning benchmarks,
including MATH500, AMC23, AIME24, AIME25~\citep{lightman2024let}, and Minerva Math~\citep{lewkowycz2022solving}, all of which are publicly available at
\url{https://huggingface.co/math-ai/datasets}.

\subsection{Training Setup}
Across all group-based RLVR methods, we remove the KL penality by setting $\beta=0$ in the loss function~\citep{liu2025understanding, yu2026dapo},
and set the number of sampled rollouts to 16 for each prompt during training for advantage estimation.
For rollout generation, we set a temperature to 1.0 with $top\_p=1.0$ and $top\_k=-1$. For evaluation, we set a temperature to 0.6 with $top\_p=0.7$ and the same $top\_k$ values.
All evaluation metrics reported are averaged over 5 different runs.
We set a training batch size of 128 and a mini-batch size of 16 with a maximum response length of 8196 tokens.
For the optimizer, we use AdamW with a learning rate of $1e-6$, optimizer parameters of $\beta=(0.9, 0.999)$, and a weight decay of $0.01$.
For GRPO~\citep{shao2024deepseekmath} and Dr.GRPO~\citep{liu2025understanding}, we set the clipping parameter $\epsilon=0.2$;
for DAPO~\citep{yu2026dapo}, we set the clipping parameter $\epsilon_{\text{low}}=0.2$ and $\epsilon_{\text{high}}=0.28$.
All experiments are conducted on 8 NVIDIA A100 GPUs.

\subsection{Hyperparameter Setup}
For hyperparameters of \algname\, we set $\alpha=0.25$ in Eq.~\eqref{eq:fgw_state}, and the barycenter size $l_b=32$.
For the entropy regularization, we set the its coefficient to $0.05$. Such hyperparameter setup is consistent across all experiments.
\subsection{Setup for Interpretability Study in Figure~\ref{fig:obs}}\label{app:exp_case_setup}
In the interpretability study in Figure~\ref{fig:vis-p-map} and~\ref{fig:loro-ot-compare}, we set the number of sampled rollouts within each outcome-grounded group to $K=16$ for computing the barycenters.
The barycenter sizes of all OT variants are set to $l_b=32$, and the entropy regularization coefficient is set to $0.05$.
For the FGW formulation, the hyperparameter $\alpha$ is set to $0.25$. All experiments in Figure~\ref{fig:obs} are conducted on the Qwen3-8B-Base~\citep{yang2025qwen3}.
The prompt for GPT-5-Sol to derive the reasoning stages of each rollout and classify tokens to these stages is included below.

\begin{tcolorbox}[
        enhanced,
        breakable,
        colback=white,
        colframe=black!75,
        colbacktitle=black!75,
        coltitle=white,
        title={Structured Reasoning-Stage Annotation Prompt},
        fonttitle=\bfseries\large,
        arc=2mm,
        boxrule=0.8pt,
        left=3mm,
        right=3mm,
        top=2mm,
        bottom=2mm
    ]
    \small
    You are given one problem and multiple reasoning rollouts.

    Your task is to infer a \textbf{shared superset of abstract reasoning stages} that collectively explains the reasoning performed across all provided rollouts, and then assign every token in every rollout to one of these stages.

    The stages should describe \textbf{reasoning functions or roles}, rather than surface wording. Examples include understanding or reformulating the problem, establishing an intermediate result, applying a theorem or relation, case analysis, algebraic simplification, verification, and deriving the final answer.

    \medskip
    \textbf{\large Shared Stage Superset}

    Infer a stage vocabulary
    \[
        \mathcal{C}=\{C_1,C_2,\ldots,C_K\}
    \]
    that is sufficiently expressive to accommodate the reasoning content appearing anywhere in the rollout group.

    The stage vocabulary is a \textbf{superset}: a stage does not need to appear in every rollout.

    For example, if some rollouts solve the problem directly while others introduce an auxiliary lemma, the shared vocabulary should contain an \texttt{Auxiliary Lemma} stage even though the direct rollouts skip it.

    Different rollouts may therefore:
    \begin{itemize}[nosep,leftmargin=*,topsep=2pt]
        \item skip one or more stages;
        \item repeat a stage;
        \item use only a subset of the shared stages;
        \item realize the same stage using substantially different wording;
        \item execute shared stages in different orders when their reasoning genuinely differs.
    \end{itemize}

    Do \textbf{not} require every rollout to instantiate every stage.

    At the same time, avoid unnecessarily creating rollout-specific stages. If two spans perform the same underlying reasoning function despite different wording or mathematical expressions, assign them to the same abstract stage.

    \medskip
    \textbf{\large Complete Token Coverage}

    Every token in every rollout must be assigned to one of the inferred reasoning stages.

    \textbf{Do not use \texttt{OTHER}, \texttt{UNKNOWN}, or any unassigned category.}

    If some tokens do not fit the current stage vocabulary, refine or expand the shared stage vocabulary so that they can be assigned to an appropriate reasoning stage.

    This includes transitional language, equations, mathematical symbols, intermediate calculations, verification statements, and final-answer text. Assign such tokens to the reasoning stage whose function they support.

    Prefer a compact and meaningful stage vocabulary over creating stages for individual words or superficial formatting differences.

    \medskip
    \textbf{\large Annotation Procedure}
    \begin{enumerate}[leftmargin=*,topsep=2pt,itemsep=2pt]
        \item Read the problem and \textbf{all rollouts jointly} before defining the stages.
        \item Identify the union of distinct reasoning functions required to explain the rollouts.
        \item Consolidate semantically equivalent reasoning functions into shared stages even when their surface realizations differ.
        \item Create additional stages when necessary to accommodate reasoning strategies that occur only in a subset of rollouts.
        \item For each rollout, segment the complete token sequence into contiguous spans.
        \item Assign every span to exactly one stage from the shared stage vocabulary.
        \item Preserve genuine differences in reasoning order. Do not relabel spans merely to make all rollouts follow the same stage sequence.
    \end{enumerate}

    \medskip
    \textbf{\large Output Requirements}

    For each stage, provide:
    \begin{itemize}[nosep,leftmargin=*,topsep=2pt]
        \item \texttt{stage\_id}
        \item \texttt{stage\_name}
        \item \texttt{definition}
    \end{itemize}

    For each rollout, provide contiguous token spans with:
    \begin{itemize}[nosep,leftmargin=*,topsep=2pt]
        \item \texttt{start\_token}
        \item \texttt{end\_token}
        \item \texttt{stage\_id}
    \end{itemize}

    Use the supplied token indices exactly. Do not invent, modify, or renumber tokens.

    A stage may occur multiple times in a rollout, in which case multiple spans may have the same \texttt{stage\_id}.

    \medskip
    \textbf{\large Output Format}

    Return \textbf{valid JSON only} using the following schema:

    \begin{Verbatim}[fontsize=\scriptsize,breaklines=true,breakanywhere=true]
        {
            "stages": [
                {
                    "stage_id": "C1",
                    "stage_name": "...",
                    "definition": "..."
                },
                {
                    "stage_id": "C2",
                    "stage_name": "...",
                    "definition": "..."
                }
            ],
            "rollouts": [
                {
                    "rollout_id": "R1",
                    "spans": [
                        {
                            "start_token": 0,
                            "end_token": 12,
                            "stage_id": "C1"
                        },
                        {
                            "start_token": 13,
                            "end_token": 27,
                            "stage_id": "C3"
                        }
                    ]
                }
            ]
        }
    \end{Verbatim}

    \medskip
    \textbf{\large Final Validation}

    Before returning the JSON, verify all of the following:
    \begin{itemize}[nosep,leftmargin=*,topsep=2pt,itemsep=2pt]
        \item Every rollout token is assigned exactly once.
        \item The first span begins at the first token.
        \item The final span ends at the final token.
        \item Consecutive spans are contiguous with no gaps.
        \item Spans never overlap.
        \item Every assigned \texttt{stage\_id} exists in \texttt{stages}.
        \item No \texttt{OTHER}, \texttt{UNKNOWN}, or unlabeled tokens are present.
        \item Stages form a \textbf{superset across the rollout group}; individual rollouts are allowed to skip stages.
        \item Stages capture reasoning functions rather than superficial wording, token type, or formatting.
        \item The stage vocabulary is no larger than necessary to explain meaningful reasoning differences across the rollouts.
    \end{itemize}

    \textbf{Return JSON only.}
\end{tcolorbox}

\section{Additional Experiment Results}\label{app:add_exp_res}
\subsection{Additional Benchmark Results}\label{app:exp_add_bench}

\begin{table*}[t]
    \centering
    \caption{Additional benchmark results of credit assignment methods on Dr.GRPO and DAPO with Qwen3-8B-Base.
        The 1st/2nd best results are highlighted in \textbf{bold}/\underline{underline}, respectively.
        Our methods are highlighted in \colorbox{red!8}{\!red box\!}.
    }
    \label{tab:bench_drgrpo_dapo}
    \vspace{-8pt}
    \begingroup
    \newcommand{\appendixbenchstd}[1]{\,{\scriptsize\textcolor{gray}{$\pm$#1}}}
    \setlength{\tabcolsep}{8pt}
    \renewcommand{\arraystretch}{1.15}
    \resizebox{\linewidth}{!}{%
        \begin{tabular}{l|ccccc|c}
            \noalign{\hrule height 1pt}
            \textbf{Method}    & \textbf{MATH500}                       & \textbf{Minerva Math}                  & \textbf{AMC23}                         & \textbf{AIME24}                        & \textbf{AIME25}                        & \textbf{Average} \\
            \hline
            \multicolumn{7}{c}{\rule{0pt}{2.5ex}\textbf{\textit{Qwen3-8B-Base}}}                                                                                                                                                                               \\
            \hline
            Dr.GRPO            & 80.1\appendixbenchstd{3.1}             & 33.8\appendixbenchstd{2.5}             & 87.0\appendixbenchstd{1.7}             & 53.3\appendixbenchstd{0.9}             & 42.0\appendixbenchstd{0.6}             & 59.2             \\
            ~~+~GRPO-$\lambda$ & 81.5\appendixbenchstd{1.7}             & 34.1\appendixbenchstd{0.3}             & 87.0\appendixbenchstd{0.9}             & 50.7\appendixbenchstd{0.6}             & 39.3\appendixbenchstd{0.9}             & 58.5             \\
            ~~+~Entropy Adv.   & 82.8\appendixbenchstd{2.8}             & 35.4\appendixbenchstd{1.9}             & \textbf{88.5}\appendixbenchstd{0.7}    & 52.7\appendixbenchstd{0.6}             & \underline{44.7}\appendixbenchstd{1.3} & 60.8             \\
            ~~+~SHEAR          & 82.6\appendixbenchstd{1.5}             & \underline{35.9}\appendixbenchstd{2.1} & 86.0\appendixbenchstd{1.3}             & 54.7\appendixbenchstd{1.3}             & \textbf{45.3}\appendixbenchstd{1.1}    & 60.9             \\
            \rowcolor{red!8}
            ~~+~\algname-seq   & \underline{83.2}\appendixbenchstd{1.4} & 35.2\appendixbenchstd{1.0}             & \underline{87.5}\appendixbenchstd{1.1} & \underline{55.3}\appendixbenchstd{1.1} & \underline{44.7}\appendixbenchstd{0.6} & 61.2             \\
            \rowcolor{red!8}
            ~~+~\algname-tok   & \textbf{84.3}\appendixbenchstd{1.8}    & \textbf{36.1}\appendixbenchstd{1.2}    & \textbf{88.5}\appendixbenchstd{0.4}    & \textbf{56.7}\appendixbenchstd{0.6}    & \underline{44.7}\appendixbenchstd{0.9} & \textbf{62.1}    \\
            \hline
            \multicolumn{7}{c}{\rule{0pt}{2.5ex}\textbf{\textit{Qwen3-8B-Base}}}                                                                                                                                                                               \\
            \hline
            DAPO               & 82.1\appendixbenchstd{2.8}             & 35.8\appendixbenchstd{2.7}             & 89.0\appendixbenchstd{0.9}             & 52.0\appendixbenchstd{1.1}             & 42.0\appendixbenchstd{0.9}             & 60.2             \\
            ~~+~GRPO-$\lambda$ & 81.8\appendixbenchstd{1.9}             & 35.2\appendixbenchstd{1.7}             & 88.0\appendixbenchstd{0.9}             & 51.3\appendixbenchstd{0.6}             & 40.7\appendixbenchstd{0.6}             & 59.4             \\
            ~~+~Entropy Adv.   & 83.9\appendixbenchstd{2.0}             & 36.1\appendixbenchstd{1.9}             & \underline{90.5}\appendixbenchstd{1.1} & 52.7\appendixbenchstd{0.9}             & 42.7\appendixbenchstd{1.1}             & 61.2             \\
            ~~+~SHEAR          & 82.7\appendixbenchstd{1.6}             & \textbf{37.2}\appendixbenchstd{1.3}    & 90.0\appendixbenchstd{1.4}             & \underline{55.3}\appendixbenchstd{0.6} & 44.0\appendixbenchstd{0.6}             & 61.9             \\
            \rowcolor{red!8}
            ~~+~\algname-seq   & \textbf{85.2}\appendixbenchstd{1.7}    & 36.5\appendixbenchstd{2.1}             & 89.5\appendixbenchstd{0.4}             & \textbf{56.7}\appendixbenchstd{0.6}    & \underline{44.7}\appendixbenchstd{0.6} & 62.5             \\
            \rowcolor{red!8}
            ~~+~\algname-tok   & \underline{85.1}\appendixbenchstd{1.0} & \underline{37.1}\appendixbenchstd{0.9} & \textbf{91.5}\appendixbenchstd{1.3}    & \textbf{56.7}\appendixbenchstd{0.9}    & \textbf{45.3}\appendixbenchstd{0.6}    & \textbf{63.1}    \\
            \noalign{\hrule height 1pt}
        \end{tabular}%
    }
    \endgroup
\end{table*}

To further verify the generalizability of \algname\ on different group-based RLVR methods,
we conduct additional benchmark experiments on Dr.GRPO~\citep{liu2025understanding} and DAPO~\citep{yu2026dapo} using Qwen3-8B-Base.
The results are shown in Table~\ref{tab:bench_drgrpo_dapo}, which demonstrate that \textbf{1)} \algname\ consistently improves the performance of Dr.GRPO and DAPO across all datasets,
and \textbf{2)} \algname\ predominantly outperforms existing credit assignment methods on most datasets and LLMs.
Both observations demonstrate the effectiveness and generalizability of \algname\ in improving group-based RLVR methods by identifying and emphasizing novel reasoning elements during RL.

\subsection{Additional Studies}\label{app:exp_add_study}

\begin{center}
    \begin{minipage}[t]{0.40\linewidth}
        \vspace{0pt}
        \centering
        \includegraphics[width=\linewidth]{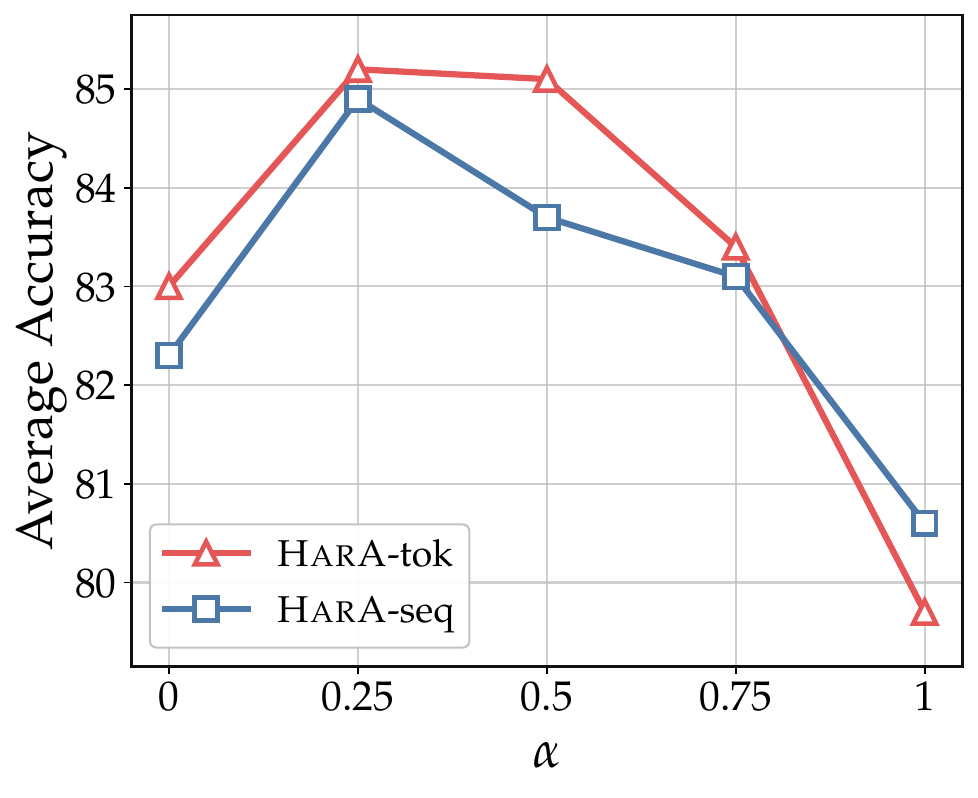}
        \vspace{-20pt}
        \captionof{figure}{Parameter study on $\alpha$ by Qwen3-8B-Base on MATH500.}
        \label{fig:exp_alpha}
    \end{minipage}
    \hfill
    \begin{minipage}[t]{0.58\linewidth}
        \vspace{0pt}
        \centering
        \captionof{table}{Ablation study of anchor-guided linearization by Qwen3-8B-Base on MATH500.}
        \label{tab:ablation_agl}
        \vspace{-6pt}
        \renewcommand{\arraystretch}{1.3}
        \resizebox{\linewidth}{!}{%
            \begin{tabular}{l|cc}
                \noalign{\hrule height 1pt}
                \textbf{Method}       & \multicolumn{2}{c}{\textbf{MATH500}}                        \\
                \cline{2-3}
                                      & \textbf{Avg. Accuracy}               & \textbf{Time / Step} \\
                \hline
                \algname(w/o AGL)-seq & 84.4                                 & 154                  \\
                \algname(w AGL)-seq   & 84.7                                 & 41.0                 \\
                \rowcolor{blue!8}
                $\Delta$              & $-0.3$                               & $3.75\times$         \\
                \noalign{\hrule height 0.6pt}
                \algname(w/o AGL)-tok & 85.3                                 & 175                  \\
                \algname(w AGL)-tok   & 85.2                                 & 43.6                 \\
                \rowcolor{blue!8}
                $\Delta$              & $+0.1$                               & $4.01\times$         \\
                \noalign{\hrule height 1pt}
            \end{tabular}%
        }
    \end{minipage}
\end{center}

\paragraph{Parameter Study on $\alpha$.}
We conduct a parameter study on the parameter $\alpha$ in Eq.~\eqref{eq:fgw_dist} which controls the trade-off between the Wasserstein distance and Gromov-Wasserstein distance in the FGW formulation.
The results on MATH500 with Qwen3-8B-Base by GRPO are shown in Figure~\ref{fig:exp_alpha}, where we can see that \textbf{the performance of \algname\ remain relatively stable when $\alpha$ is in the range of $\mathbf{[0.25,0.75]}$},
with less than 2\% performance variation. If $\alpha=0$ or $1$, the FGW formulation reduces to the Wasserstein distance or Gromov-Wasserstein distance, respectively,
which fails to capture either the semantic or structural information of a rollout, leading to more severe performance degradation. We set $\alpha=0.25$ as the default value in \algname.

\paragraph{Ablation on Anchor-Guided Linearization.}
To further understand how the anchor-guided linearization in Lemma~\ref{lem:fgw_linear} affects the empricial performance and efficiency of \algname,
we conduct an ablation study on MATH500 with Qwen3-8B-Base by GRPO, comparing the performance of \algname\ with w/ and w/o the anchor-guided linearization (AGL),
as shown in Table~\ref{tab:ablation_agl}. We can observe that \textbf{1) the anchor-guided linearization hardly affect the performance of \algname}, with less than 1\% performance variation.
Such observation is consistent with our theoretical analysis in Lemma~\ref{lem:fgw_linear}, where we show that the approximation error of the linearization is tightly bounded.
Meanwhile, \textbf{2) the anchor-guided linearization significantly improves the efficiency of \algname} thanks to strict convexity of the linearized problem~\cite{gabriel2019computational},
with up to 4$\times$ speed-up in training time per gradient update step. This futher justify using the anchor-guided linearization in \algname\ to speed-up the OT computation in \algname.

\end{document}